%% file: paper_revision.tex
\documentclass[twoside,11pt]{article}
\usepackage{jmlr2e}

\RequirePackage{amsmath}
\usepackage[dvipsnames]{xcolor}

\newtheorem{algorithm}{Algorithm}
\newtheorem{theorem}{Theorem}
\newtheorem{lemma}{Lemma}
\newtheorem{assumption}{A.\!\!}
\newtheorem{definition}{Definition}
\newtheorem{remark}{Comment}[section]
\jmlrheading{1}{2022}{XXX}{XXX}{XXX}{}{Jannis Kueck, Ye Luo and Martin Spindler}

\ShortHeadings{High-Dimensional $L_2$-Boosting}{Kueck, Luo and Spindler}

\firstpageno{1}

\begin{document}

\title{High-Dimensional $L_2$-Boosting: Rate of Convergence}

\author{\name Ye Luo \email kurtluo@hku.hk\\
       \addr Hong Kong University Business School\\
       The University of Hong Kong\\
       Hong Kong
       \AND
       \name Martin Spindler \email martin.spindler@uni-hamburg.de \\
       \addr Institute for Statistics\\
       University of Hamburg\\
       Germany
       \AND
       \name Jannis Kueck \email jannis.kueck@uni-hamburg.de \\
       \addr Institute for Statistics\\
       University of Hamburg\\
       Germany}

\editor{N.N.}

\maketitle

\begin{abstract}
Boosting is one of the most significant developments in machine learning. This paper studies the rate of convergence of $L_2$Boosting, which is tailored for regression, in a high-dimensional setting. Moreover, we introduce so-called  \textquotedblleft post-Boosting\textquotedblright. This is a post-selection estimator which applies ordinary least squares to the variables selected in the first stage by $L_2$Boosting. Another  variant is \textquotedblleft Orthogonal Boosting\textquotedblright where after each step an orthogonal projection is conducted. We show that both post-$L_2$Boosting and the orthogonal boosting achieve the same rate of convergence as LASSO in a sparse, high-dimensional setting. We show that the rate of convergence of the classical $L_2$Boosting depends on the design matrix described by a sparse eigenvalue constant.
To show the latter results, we derive new approximation results for the pure greedy algorithm, based on analyzing the revisiting behavior of $L_2$Boosting.
We also introduce feasible rules for early stopping, which can be easily implemented and used in applied work. Our results also allow a direct comparison between LASSO and boosting which has been missing from the literature. Finally, we present simulation studies and applications to illustrate the relevance of our theoretical results and to provide insights into the practical aspects of boosting. In these simulation studies, post-$L_2$Boosting clearly outperforms LASSO.
\end{abstract}


\section{Introduction}

In this paper we consider $L_2$Boosting algorithms for regression which are coordinatewise greedy algorithms that estimate the target function under $L_2$ loss. Boosting algorithms represent one of the major advances in machine learning and statistics in recent years. Freund and Schapire's AdaBoost algorithm for classification (\cite{freund:1997})  has attracted much attention in the machine learning community as well as in statistics. Many variants of the AdaBoost algorithm have been introduced and proven to be very competitive in terms of prediction accuracy in a variety of applications with a strong resistance to overfitting. Boosting methods were originally proposed as ensemble methods, which rely on the principle of generating multiple predictions and majority
voting (averaging) among the individual classifiers (cf \cite{nr:buhlmann.hothorn:2007}). An important step in the analysis of Boosting algorithms was Breiman's interpretation of Boosting  as a
gradient descent algorithm in function space, inspired
by numerical optimization and statistical estimation (\cite{nr:breiman:1996}, \cite{nr:breiman:1998}). Building on this insight, \cite{nr:friedman.hastie.ea:2000} and \cite{nr:friedman:2001} embedded Boosting algorithms into the framework of statistical estimation and additive basis expansion. This also enabled the application of boosting for regression analysis. Boosting for regression was proposed by \cite{nr:friedman:2001}, and then \cite{nr:buhlmann.yu:2003} defined and introduced $L_2$Boosting. An extensive overview of the development of Boosting and its manifold applications is given in the survey \cite{nr:buhlmann.hothorn:2007}.

In the high-dimensional setting there are two important but unsolved problems on $L_2$Boosting. First, the convergence rate of the $L_2$Boosting has not been thoroughly analyzed. Second, the pattern of the variables selected at each step of $L_2$Boosting is unknown.

In this paper, we show that these two problems are closely related. We establish results on the sequence of variables that are selected by $L_2$Boosting. At any step of $L_2$Boosting, we call this step "revisiting" if the variable chosen in this step has already been selected in previous steps. We analyze the revisiting behavior of $L_2$Boosting, i.e., how often $L_2$Boosting revisits. We then utilize these results to  derive an upper bound of the rate of convergence of the $L_2$Boosting.\footnote{Without analyzing the sequence of variables selected at each step of $L_2$Boosting, only much weaker results on convergence speed of $L_2$Boosting are available based on \cite{DeVore1996} and \cite{Livshitz2003}.} We show that frequency of revisiting, as well as the convergence speed of $L_2$Boosting, depend on the structure of the design matrix, namely on a constant related to the minimal and maximal restricted eigenvalue. Our bounds on convergence rate of $L_2$Boosting are in general slower than that of $LASSO$.

We also introduce in this paper the so--called \textquotedblleft post-Boosting\textquotedblright, and the orthogonal boosting variant.\footnote{Orthogonal boosting has also similarities with forward step-wise regression.}  
For orthogonal boosting see also ~Section 3.1 in \citet{LaiYuan2021}.
We show that both algorithms achieve the same rate of convergence as LASSO in a sparse, high-dimensional setting.

Compared to LASSO, boosting uses a somewhat unusual penalization scheme. The penalization is done by \textquotedblleft early stopping\textquotedblright\ to avoid overfitting in the high-dimensional case. In the low-dimensional case, $L_2$Boosting without stopping converges to the ordinary least squares (OLS) solution. In a high-dimensional setting early stopping is key for avoiding overfitting and for the predictive performance of boosting. We give a new stopping rule that is simple to implement and also works very well in practical settings as demonstrated in the simulation studies. We prove that such a stopping rule achieves the best bound obtained in our theoretical results.

In a deterministic setting, which is when there is no noise or error term in the model, boosting methods are also known as greedy algorithms (the pure greedy algorithm (PGA) and the orthogonal greedy algorithm (OGA)). In signal processing, $L_2$Boosting is essentially the same as the matching pursuit algorithm of \cite{ms:mallatzhang:1993}. We will employ the abbreviations post-BA (post-$L_2$Boosting algorithm) and oBA (orthogonal $L_2$Boosting algorithm) for the stochastic versions we analyze.

The rate of convergence of greedy algorithms has been analyzed in \cite{DeVore1996} and \cite{Livshitz2003}. \cite{Temlyakov:2011} is an excellent survey of recent results on the approximation theory of greedy approximation. To the best of our knowledge, with an additional assumption on the design matrix, we establish the first results on revisiting in the deterministic setting and greatly improve the existing results of \cite{DeVore1996}. These  results, being available in the appendix, are essential for our analysis for $L_2$Boosting, but might also be of interest in their own right.

As mentioned above, Boosting for regression was introduced by \cite{nr:friedman:2001}. $L_2$-Boosting was defined in \cite{nr:buhlmann.yu:2003}. Its numerical convergence, consistency, and statistical rates of convergence of boosting with early stopping in a low-dimensional setting were obtained in \cite{nr:zhang.yu:2005}. Consistency in prediction norm of $L_2$Boosting in a high-dimensional setting was first proved in \cite{nr:buhlmann:2006}. The numerical convergence properties of Boosting in a low-dimensional setting are analyzed in \cite{FGM15}. The orthogonal Boosting algorithm in a statistical setting under different assumptions is analyzed in \cite{ms:inglai:2011}. The rates for the PGA and OGA cases are obtained in \cite{ms:barron:2008}.

In this paper we consider linear basis functions. Classification  and  regression  trees,  and  the  widely  used  neural  networks,  involve  non-linear basis functions. We hope that our results can serve as a starting point for the analysis of non-linear basis functions which is left for future research.

The structure of this paper is as follows: In Section 2 the $L_2$Boosting algorithm (BA) is defined together with its modifications, the post-$L_2$Boosting algorithm (post-BA) and the orthogonalized version (oBA). In Section 3 we present a new approximation result for the pure greedy algorithm (PGA later) and an analysis of the revisiting behavior of the boosting algorithm. In Section 4 we present the main results of our analysis, namely an analysis of the boosting algorithm and some of its variants. The proofs together with some details of the new approximation theory for PGA are provided in the Appendix. Section 5 contains a simulation study that offers some insights into the methods and also provides some guidance for stopping rules in applications. Section 6 discusses two applications and provides concluding observations.


\textbf{Notation:} Let $z$ and $y$ be $n$-dimensional vectors. We define the empirical $L_2$-norm as $\mathbb{E}_n[z]=1/n \sum_{i=1}^n z_i$. Define $\lvert \lvert z\rvert\rvert$ to be the Euclidean norm, and $\lvert \lvert z\rvert\rvert_{2,n}:= \sqrt{\mathbb{E}_n[z^2]}$. Define $<\cdot,\cdot>_n$ to be the inner product defined by: $<z,y>_n=1/n \sum_{i=1}^n z_i y_i$.

For a random variable $X$, $\mathbb{E}[X]$ denotes its expectation. The correlation between the random variables $X$  and $Y$ is denoted by $corr(X,Y)$.

We use the notation $a\vee b=\max\{a,b\}$ and $a \wedge b=\min\{a,b\}$. We also use the notation $a \precsim b$ to mean $a \leq cb$ for some constant $c>0$ that does not depend on $n$; and $a \precsim_P b$ to mean $a=\mathcal{O}_P(b)$. For a set $U$, $supp(U)$ denotes the set of indices of which the corresponding element in $U$ is not zero. Given a vector $\beta \in \mathbb{R}^p$ and a set of indices $T \subset \{1,\ldots,p\}$, we denote by $\beta_T$ the vector in which $\beta_{T_j}=\beta_j$ if $j \in T$, $\beta_{T_j}=0$ if $j \notin T$.

\section{$L_2$-Boosting with componentwise least squares}
To define the boosting algorithm for linear models, we consider the following regression setting:
\begin{equation} y_i=x_i'\beta+\varepsilon_i, \quad i=1,\ldots,n, \end{equation}
with vector $x_i=(x_{i,1},\ldots,x_{i,p_n})$ consisting of $p_n$ predictor variables, $\beta$ a $p_n$-dimensional coefficient vector, and a random, mean-zero error term $\varepsilon_i$, $\mathbb{E}[\varepsilon_i|x_i]=0$. Further assumptions will be employed in the next sections.

We allow the dimension of the predictors $p_n$ to grow with the sample size $n$, and is even larger than the sample size, i.e., $dim(\beta)=p_n\gg n$. But we will impose a sparsity condition. This means that there is a large set of potential variables, but the number of variables which have non-zero coefficients, denoted by $s$, is small compared to the sample size, i.e.~ $s \ll n$. This can be weakened to approximate sparsity, to be defined and explained later. More precise assumptions will also be made later.
In the following, we will drop the dependence of $p_n$ on the sample size and denote it by $p$ if no confusion will arise.

$X$ denotes the $n \times p$ design matrix where the single observations $x_i$ form the rows. $X_j$ denotes the $j$th column of design matrix, and $x_{i,j}$ the $j$th component of the vector $x_i$.
We consider a fixed design for the regressors. We assume that the regressors are standardized with mean zero and variance one, i.e., $\mathbb{E}_n[x_{i,j}]=0$ and $\mathbb{E}_n[x_{i,j}^2]=1$ for $j=1,\ldots,p$,


The basic principle of Boosting can be described as follows. We follow the interpretation of \cite{nr:breiman:1998} and \cite{nr:friedman:2001} of Boosting as a functional gradient descent optimization (minimization) method. The goal is to minimize a loss function, e.g., an $L_2$-loss or the negative log-likelihood function of a model, by an iterative optimization scheme. In each step the (negative) gradient which is used in every step to update the current solution is modelled and estimated by a parametric or nonparametric statistical model, the so-called base learner. The fitted gradient is used for updating the solution of the optimization problem. A strength of boosting, besides the fact that it can be used for different loss functions, is its flexibility with regard to the base learners. We then repeat this procedure until some stopping criterion is met.

The literature has developed many different forms of boosting algorithms. In this paper we consider $L_2$Boosting with componentwise linear least squares, as well as two variants. All three are designed for regression analysis. \textquotedblleft $L_2$\textquotedblright refers to the loss function, which is the typical sum-of-squares of the residuals $Q_n(\beta)=\sum_{i=1}^n (y_i - x_i' \beta)^2$ typical in regression analysis. In this case, the gradient equals the residuals. \textquotedblleft Componentwise linear least squares\textquotedblright refers to the base learners. We fit the gradient (i.e. residuals) against each regressor ($p$ univariate regressions) and select the predictor/variable which correlates most highly with the gradient/residual, i.e., decreases the loss function most, and then update the estimator in this direction. We next update the residuals and repeat the procedure until some stopping criterion is met. We consider $L_2$Boosting and two modifications: the \textquotedblleft classical\textquotedblright one which was introduced in \cite{nr:friedman:2001} and refined in \cite{nr:buhlmann.yu:2003} for regression analysis, an orthogonal variant and post-$L_2$Boosting. As far as we know, post-$L_2$Boosting has not yet been defined and analyzed in the literature. In signal processing and approximation theory, the first two methods are known as the pure greedy algorithm (PGA) and the orthogonal greedy algorithm (OGA) in the deterministic setting, i.e. in a setting without stochastic error terms.

\subsection{$L_2$Boosting}
For $L_2$Boosting with componentwise least squares, the algorithm is given below.

\begin{algorithm}[$L_2$-Boosting]
\begin{enumerate}
    \item Start / Initialization: $\beta^0 = 0$ ($p$-dimensional vector), $f^0=0$, set maximum number of iterations $m_{stop}$ and set iteration index $m$ to $0$.
    \item At the $(m+1)^{th}$ step, calculate the residuals $U_i^m=y_i - x_i' \beta^m$.
    \item For each predictor variable $j=1,\ldots,p$ calculate the correlation with the residuals:
    \begin{equation*} \gamma^m_{j}:=\frac{\sum_{i=1}^n U_i^m x_{i,j}}{\sum_{i=1}^n x_{i,j}^2}=\frac{<U^m,x_j>_n}{\mathbb{E}_n[x_{i,j}^2]}.  \end{equation*}
    Select the variable $j^m$ that is the most correlated with the residuals\footnote{Equivalently, which fits the gradient best in a $L_2$-sense.}, i.e., $$\max_{1\leq j\leq p}|corr(U^{m},x_{j})|.$$
    \item Update the estimator: $\beta^{m+1}:=\beta^{m}+\gamma^m_{j^m} e_{j^m}$ where $e_{j^m}$ is the $j^m$th index vector
and
    $f^{m+1}:=f^{m}+\gamma^m_{j^m}x_{j^m}$
 \item Increase $m$ by one. If $m<m_{stop}$, continue with (2); otherwise stop.
\end{enumerate}
\end{algorithm}

For simplicity, write $\gamma^m$ for the value of $\gamma^m_{j^m}$ at the $m^{th}$ step.

The act of stopping is crucial for boosting algorithms, as stopping too late or never stopping leads to overfitting and therefore some kind of penalization is required. A suitable solution is to stop early, i.e., before overfitting takes place. \textquotedblleft Early stopping\textquotedblright\ can be interpreted as a form of penalization. Similar to LASSO, early stopping might induce a bias through shrinkage. A potential way to decrease the bias is by \textquotedblleft post-Boosting\textquotedblright which is defined in the next section.

In general, during the run of the boosting algorithm, it is possible that the same variable is selected at different steps, which means the variable is revisited. This revisiting behavior is key to the analysis of the rate of convergence of $L_2$Boosting. In the next section we will analyze the revisting properties of boosting in more detail.

\subsection{Post-$L_2$Boosting}

Post-$L_2$Boosting is a post-model selection estimator that applies ordinary least squares (OLS) to the model selected by the first-step, which is $L_2$Boosting. To define this estimator formally, we make the following definitions: $T:=supp(\beta)$ and $\hat{T}:=supp(\beta^{m^*})$, the support of the true model and the support of the model estimated by $L_2$Boosting as described above with stopping at $m^*$. A superscript $C$ denotes the complement of the set with regard to $\{1,\ldots,p\}$. In the context of LASSO, OLS after model selection was analyzed in \cite{belloni:2013}.
Given the above definitions, the post-model selection estimator or OLS post-$L_2$Boosting estimator will take the form
\begin{equation}
\tilde{\beta}= arg min_{\beta \in \mathbb{R}^p}\ Q_n(\beta): \beta_j=0 \quad \mbox{for each} \quad j \in \hat{T}^C.
\end{equation}

\begin{remark}
For boosting algorithms it has been recommended -- supported by simulation studies -- not to update by the full step size $x_{j^m}$ but only a small step $\nu$. The parameter $\nu$ can be interpreted as a shrinkage parameter, or alternatively, as describing the step size when updating the function estimate along the gradient. Small step sizes (or shrinkage) make the boosting algorithm slower to converge and require a larger number of iterations. But often the additional computational cost in turn results in better out-of-sample prediction performance. By default, $\nu$ is usually set to $0.1$. Our analysis in the later sections also extends to a restricted step size $0<\nu<1$.

\end{remark}
\subsection{Orthogonal $L_2$Boosting}
A variant of the Boosting Algorithm is orthogonal Boosting (oBA) or the Orthogonal Greedy Algorithm in its deterministic version. Only the updating step is changed: an orthogonal projection of the response variable is conducted on all the variables which have been selected up to this point. The advantage of this method is that any variable is selected at most once in this procedure, while in the previous version the same variable might be selected at different steps which makes the analysis far more complicated.
More formally, the method can be described as follows by modifying Step (4):
\begin{algorithm}[Orthogonal $L_2$Boosting]
\begin{equation*} (4') \quad \hat{y}^{m+1} \equiv f^{m+1} = P_m y \quad \text{and} \quad U_i^{m+1}=Y_i-\hat{Y}_i^{m+1},\end{equation*}
where $P_m$ denotes the projection of the variable $y$ on the space spanned by first $m$ selected variables (the corresponding regression coefficient is denoted $\beta^{m}_o$.)
\end{algorithm}

Define $X_o^m$ as the matrix which consists only of the columns which correspond to the variables selected in the first $m$ steps, i.e. all $X_{j_k}$, $k=0,1,\ldots,m$.
Then we have:
\begin{eqnarray}
 \beta^{m}_o &=& ({X_o^m}' X_o^m)^{-1} {X_o^m}' y\\
\hat{y}^{m+1}=f_o^{m+1}&=& X_o^m \beta^{m}_o
\end{eqnarray}

\begin{remark}
Orthogonal $L_2$Boosting might be interpreted as post-$L_2$Boosting where the refit takes place after each step.
\end{remark}

\begin{remark}
Both post-Boosting and orthogonal Boosting require, to be well-defined, that the number of selected variables be smaller than the sample size . This is enforced by our stopping rule as we will see later.
\end{remark}

\section{New Approximation Results for the Pure Greedy Algorithm}

In approximation theory a key question is how fast functions can be approximated by greedy algorithms. Approximation theory is concerned with deterministic settings, i.e., the case without noise. Nevertheless, to derive rates for the $L_2$Boosting algorithm in a stochastic setting, the corresponding results for the deterministic part play a key role. For example, the results in \cite{nr:buhlmann:2006} are limited by the result used from approximation theory, namely the rate of convergence of weak relaxed greedy algorithms derived in \cite{Temlyakov2000}. For the pure greedy algorithm \cite{DeVore1996} establish a rate of convergence of $m^{-1/6}$ in the $\ell_2-$norm, where $m$ denotes the number of steps iterated in the PGA. This rate was improved to $m^{-11/62}$ in \cite{Konyagin1999}, but \cite{Livshitz2003} establish a lower bound of $m^{-0.27}$.
The class of functions $\mathcal{F}$ which is considered in those papers is determined by general dictionaries $\mathcal{D}$ and given by
\[
\mathcal{F} = \{ f \in \mathcal{H}: f = \sum_{k \in \Lambda} c_k w_k, w_k \in \mathcal{D}, |\Lambda| < \infty \quad \mbox{and} \quad \sum_{k \in \Lambda} |c_k| \leq M \},
\]
where $M$ is some constant, $\mathcal{H}$ denotes a Hilbert space, and the sequence $(c_k)$ are the coefficients with regard to the dictionary $\mathcal{D}$.

In this section we discuss the approximation bound of the pure greedy algorithm  where we impose an additional but widely used assumption on the Gram matrix $\mathbb{E}_n[x_i x_i']$ in high dimensional statistics to tighten the bounds. First, the assumptions and an initial result describing the revisiting behavior will be given, then a new approximation result based on the revisiting behavior will be presented which is the core of this section. The proofs for this section and a detailed analysis of the revisiting behavior of the algorithm are moved to Appendix A.

\subsection{Assumptions}

For the analysis of the pure greedy algorithm, the following two assumptions are made, which are standard for high-dimensional models.

\begin{assumption}[Exact Sparsity]\label{Sparsity}
$T=supp(\beta)$ and $s=|T|\ll n$.
\end{assumption}

\begin{remark}
The exact sparsity assumption can be weakened to an approximate sparsity condition, in particular in the context of the stochastic version of the pure greedy algorithm ($L_2$Boosting). This means that the set of relevant regressors is small, and the other variables do not have to be exactly zero but must be negligible compared to the estimation error.
\end{remark}

For the second assumption, we make a restricted eigenvalue assumption which is also commonly used in the analysis of LASSO.

Define $\Sigma(s,M):=\{A|dim(A)\leq s\times s, A \textrm{ is any diagonal submatrices of }M\}$, for any square matrix $M$.

We need the following definition.

\begin{definition}
The smallest and largest restricted eigenvalues are defined as
 $$\phi_s(s,M):=\min_{W\in \Sigma(s,M)} \phi_s(W),$$
and
$$\phi_l(s,M):=\max_{W\in \Sigma(s,M)} \phi_l(W).$$
$\phi_s(W)$ and $ \phi_l(W)$ denote the smallest and largest eigenvalue of the matrix $W$.
\end{definition}

\begin{assumption} \textbf{(SE)}\label{SE}
We assume that there exist constant $0<c<1$ and  $C$ such that $0<1-c\leq \phi_s(s',E_n[x_i'x_i])\leq \phi_l(s',E_n[x_i'x_i])\leq C < \infty$ for any $s'\leq M_0$,
where $M_0$ is a sequence such that $M_0\rightarrow \infty$ slowly along with $n$, and $M_0\geq s$.
\end{assumption}

\begin{remark}
This condition is a variant of the so-called \textquotedblleft  sparse eigenvalue condition\textquotedblright, which is used for the analysis of the Lasso estimator. A detailed discussion of this condition is given in \cite{BCH2011:InferenceGauss}. Similar conditions, such as the restricted isometry condition or the restricted eigenvalue condition, have been used for the analysis of the Dantzig Selector (\cite{candes:2007}) or the Lasso estimator (\cite{BickelRitovTsybakov2009}). An extensive overview of different conditions on matrices and how they are related is given by \cite{vandegeer2009}. To assume that $\phi_l(m, E_n[x_i x_i']) > 0$ requires that all empirical Gram submatrices formed by any $m$ components of $x_i$ are positive definite. It is well-known that Condition SE is fulfilled for many designs of
interest.
\end{remark}

More restrictive requirements that $M_0$ should be large enough will be imposed in order to get good convergence rate for the PGA, i.e., $L_2$Boosting without a noise term.

Define $V^m=X\alpha^m$ as the residual for the PGA. $\alpha^m$ is defined as the difference between the true parameter vector $\beta$ and the approximation at the $m^{th}$ step, $\beta^m$, $\alpha^m= \beta - \beta^m$. We would like to explore how fast $V^m$ converges to $0$. In our notation, $||V^{m+1}||^2_{2,n}=||V^m||_{2,n}^2-(\gamma^m)^2$, therefore $||V^m||^2_{2,n}$ is non-increasing in $m$. 

As described in Algorithm 1, the sequence of variables selected in the PGA is denoted by $j^0,j^1,\ldots$. Define $T^m:=T\cup \{j^0,j^1,\ldots,j^{m-1}\}$. Define $q(m):=|T^m|$ as the cardinality of $T^m$, $m=0,1,\ldots$. It is obvious that $q(m)\leq m+s$. 

It is essential to understand how PGA revisits the set of already selected variables. To analyze the revisiting behavior of the PGA, some definitions are needed to fix ideas.

\begin{definition}\label{revisiting:definitions}
We say that the PGA is revisiting at the $m^{th}$ step, if and only if $j^{m-1}\in T^{m-1}$. We define the sequence of labels $\mathcal{A}:=\{A_1,A_2,...\}$ with each entry $A_i$ being either labelled as $R$(revisiting) or $N$(non-revisiting).
\end{definition}

\begin{lemma}\label{Lemma:FR}
Assume that assumptions A.1-A.2 hold. Assume that $m+k<M_0$. Consider the sequence of steps $1,2,...,m$. Denote $\mu_a(c)=[1-(1+\frac{1}{(1-c)^2})^{-\frac{1}{1-c}}]$ for any $c\in (0,1)$.
Then for any $\delta>0$, the number of $R$s in the sequence $\mathcal{A}$ at step $m$, denoted $R(m)$, must satisfy:
$$|R(m)|\geq \frac{1-(1+\delta)\mu_a(c)}{2-(1+\delta)\mu_a(c)}m-\frac{(1+\delta)\mu_a(c)}{2-(1+\delta)\mu_a(c)}q(0).$$
\end{lemma}

The lower bound stated in Lemma \ref{Lemma:FR} has room for improvement, e.g., when $c=0$, $|R(m)|/m=1$ as it is shown in Lemma \ref{Lemma:simple results on revisiting} in Appendix A, while we get $1/2$ in Lemma \ref{Lemma:FR} as lower bounds of $|R(m)/m|$ as $m$ becomes large enough. Deriving tight bounds is an interesting question for future research. More detailed properties of the revisiting behavior of $L_2$Boosting are provided in the Appendix A.

\subsection{Approximation bounds on PGA}

With an estimated bound for the proportion of $R$s in the sequence $\mathcal{A}$, we are now able to derive an upper bound for $||V^m||_{2,n}^2$. By Lemma \ref{Lemma:FR}, define $n^*_k:=\frac{m+\mu_a(c)q(k)}{2-\mu_a(c)}$ which is an upper bound of $|q(m+k)-q(k)|$ up to constant converging to $1$ as $q(m)$ goes to infinity. Before we state the main result of this section we present an auxiliary lemma.

\begin{lemma}\label{lemma:appr-ori}
Let $\lambda>0$ be a constant. Let $m=\lambda q(k)$. Consider the steps numbered as $k+1,...,k+m$. Assume that $m+k<M_0$. Define $\zeta(c,\lambda):=\frac{\frac{(1-c)((1-\mu_a(c))\lambda-\mu_a(c))}{2+\lambda}}{\log(\frac{2+\lambda}{2-\mu_a(c)})}+1-c$ for all $\lambda\geq \frac{\mu_a(c)}{1-\mu_a(c)}$.

Then, for any $\delta>0$ and $q(k)>0$ large enough, the following statement holds:
$$||V^{m+k}||_{2,n}^2\leq ||V^k||_{2,n}^2 \left(\frac{q(k)}{q(k)+n^*_k}\right)^{\zeta(c,\lambda)-\delta}.$$
\end{lemma}

Based on Lemma \ref{lemma:appr-ori}, we are able to develop our main results on approximation theory of pure greedy algorithm under $L_2$ loss and Assumptions \ref{Sparsity} and \ref{SE}.

\begin{theorem}[Approximation Theory of PGA based on revisiting]\label{theorem:appr}
Define $\zeta^*(c):=\max_{\lambda\geq \frac{\mu_a(c)}{1-\mu_a(c)}} \zeta(c,\lambda)$ as a function of $c$. Then, for any $\delta>0$ and $m<M_0$, there exists a constant $C>0$ so that $||V^m||_{2,n}^2/||V^0||_{2,n}^2\leq C(\frac{s}{m+s})^{\zeta^*(c)-\delta}$ for $m$ large enough.
\end{theorem}

\begin{remark}
Our results stated in Theorem \ref{theorem:appr} depend on the lower bound of $|R(m)|/m$, which is the proportion of the $R$s in the first $m$ terms in the sequence $\mathcal{A}$. We conjecture that the convergence rate of PGA is close to exponential as $c\rightarrow 0$. Denote the actual proportion of $R$ in the sequence $\mathcal{A}$ by $\psi(c)$, i.e., $|R(m)|\geq \psi(c)m-\psi_1(c)q(0)$, where $\psi(c),\psi_1(c)$ are some constants depending on $c$. If $\psi(c)\rightarrow 1$, it is easy to show that $||V^{m}||^2_{2,n}\precsim ||V^0||^2_{2,n} \left(\frac{s}{s+m}\right)^\zeta$, based on the proof of Theorem 1, for any arbitrarily large $\zeta$. In general, further improvements of the convergence convergence rate of PGA can be achieved by improving the lower bounds of $|R(m)|/m$.

Table \ref{c_and_zeta} gives different values of the SE constant $c$ for the corresponding values of $\zeta^*$.

\begin{table}[ht]
\caption{Relation between $c$ and $\zeta$}
 \label{c_and_zeta}
\begin{tabular}{ll}
  \hline \hline
    $c$ & $\zeta^*(c)$ \\
    \hline
    $0.0$ & $1.19$\\
    $0.1$ & $1.04$\\
    $0.2$ & $0.89$\\
    $0.3$ & $0.76$\\
    $0.5$ & $0.63$\\
    $0.6$ & $0.51$\\
    $0.7$ & $0.40$\\
    \hline \hline
\end{tabular}
 \end{table}

The convergence rate of PGA and hence of $L_2$Boosting is affected by the frequency of revisiting. Because different values of $c$ impose different lower bounds on the frequencies of revisiting, thus different values of $c$ imply a different convergence rate of the process in our framework.
\end{remark}

\section{Main Results}

In this section we discuss the main results regarding the $L_2$Boosting procedure (BA), post-$L_2$Boosting (post-BA) and the orthogonal procedure (oBA) in a high-dimensional setting.

We analyze the linear regression model introduced in a high-dimensional setting, which was introduced in Section 2.

\subsection{$L_2$Boosting with Componentwise Least Squares}

First, we analyze the classical $L_2$Boosting algorithm with componentwise least squares. For this purpose, the approximation results which we derived in the previous section are key. While in the previous section the stochastic component was absent, in this section it is explicitly considered.

The following definitions will be helpful for the analysis:
$U^{m}$ denotes the residuals at the $m^{th}$ iteration, $U^{m}=Y-X\beta^{m}$. $\beta^{m}$ is the estimator at the $m^{th}$ iteration. We define the difference between the true and the estimated vector as $\alpha^{m}:=\beta-\beta^m$. The prediction error is given by $V^m=X\alpha^m$.

For the Boosting algorithm in the high-dimensional setting it is essential to determine when to stop, i.e.~the stopping criterion. In the low-dimensional case, stopping time is not important: the value of the objective function decreases and converges to the traditional OLS solution exponentially fast, as described in B\"uhlmann and Yu (2006). In the high-dimensional case, such fast convergence rates are usually not available: the residual $\varepsilon$ can be explained by $n$ linearly independent variables $x_j$. Thus selecting more terms only leads to overfitting. Early stopping is comparable to the penalization in LASSO, which prevents one from choosing too many variables and hence overfitting. Similarly to LASSO, a sparse structure will be needed for analysis.

At each step, we minimize $||U^m||_{2,n}^2$ along the \textquotedblleft most greedy\textquotedblright variable $X_{j^m}$. The next assumption is on the residual / stochastic error term $\varepsilon$ and encompasses many statistical models which are common in applied work.

\begin{assumption}
With probability greater than or equal $1-\alpha$, we have, $sup_{1\leq j\leq p} |<X_j,\varepsilon>_n|\leq 2\sigma\sqrt{\frac{\log(2p/\alpha)}{n}}:=\lambda_n$.
\end{assumption}
\begin{remark}
The previous assumption is, e.g., implied if the error terms are i.i.d. $N(0,\sigma^2)$ random variables. This in turn can be generalized / weakened to cases of non-normality by self-normalized random vector theory (\cite{delapena}) or the approach introduced in \cite{CCK:2014}.
\end{remark}

Set $\sigma_n^2:=\mathbb{E}_n[\varepsilon^2]$. Recall that $||U^{m+1}||_{2,n}^2=||U^m||_{2,n}^2-(\gamma^m_j)^2$, where $|\gamma_j^m|=\max_{1\leq j\leq p}|<X_j,U^m>_n|=\max_{1\leq j\leq p}|<X_j,V^m>_n+<X_j,\varepsilon>_n|$. The lemma below establishes the main result of convergence rate of $L_2$Boosting.

\begin{lemma}\label{Lemma:Bounds}
Suppose assumptions A.1--A.3 hold and $\frac{s\log(p)}{n} \rightarrow 0$.
Assume $M_0$ is large enough so that $\log(M_0/s)+(\xi+\frac{1}{1+\zeta^*(c)})\log(\frac{s\log(p)}{n||V^0||_{2,n}^2})>0$ for some $\xi>0$. Write ${m}^*+1$ for the first time $||V^m||_{2,n}\leq \eta\sqrt{m+s}\lambda_n$, where $\eta$ is a constant large enough. Then, for any $\delta>0$, with probability $\geq 1-\alpha$,
\begin{itemize}
\item[(1)] it holds \begin{align}\label{eq:lemma9-m}
m^*\precsim s \left(\frac{s\log(p)}{n ||V^0||_{2,n}^2}\right)^{\frac{-1}{1+\zeta^*(c)-\delta}} \quad \mbox{and}\quad m^*<M_0;
\end{align}
\item[(2)] the prediction error $||V^{m^*+1}||$ satisfies:
\begin{align}\label{eq:lemma9}
 ||V^{m^*+1}||_{2,n}^2\precsim_p ||V^0||_{2,n}^\frac{2}{1+\zeta^*(c)-\delta}\left(\frac{s\log(p)}{n}\right)^{\frac{\zeta^*(c)-\delta}{1+\zeta^*(c)-\delta}}.
\end{align}
\end{itemize}
\end{lemma}

\begin{remark}
Lemma \ref{Lemma:Bounds} shows that the convergence rate of the $L_2$Boosting depends on the value of $c$. For different values of $c$, the lower bound of the proportion of revisiting (``R'') in the sequence $\mathcal{A}$ should be different. Such lower bounds on the frequency of revisiting will naturally determine the upper bound for the deterministic component, which affects our results on the rate of convergence of $L_2$Boosting. As $\zeta^*(c)\rightarrow \infty$, the statement (2) implies the usual LASSO rate of convergence.
\end{remark}

The bound of the approximation error $||V^m||_{2,n}^2$ stated in inequality (\ref{eq:lemma9}) is obtained under an infeasible stopping criteria. Below we establish another result which employs the same convergence rate but with a feasible stopping criterion which can be implemented in empirical studies.

\begin{theorem}\label{main}
Suppose all conditions stated in Lemma \ref{Lemma:Bounds} hold.
Let $c_u>4$ be a constant. Let $m^*_1+1$ be the first time such that $\frac{||U^m||^2_{2,n}}{||U^{m-1}||^2_{2,n}} > 1-c_u\log(p)/n$. Then, with probability at least $1-\alpha$,

$||V^{m_1^*}||_{2,n}^2\precsim ||V^0||_{2,n}^\frac{2}{1+\zeta^*(c)-\delta}(\frac{s\log(p)}{n})^{\frac{\zeta^*(c)-\delta}{1+\zeta^*(c)-\delta}}$.
\end{theorem}


\begin{remark}
As we have already seen in the deterministic case, the rate of convergence depends on the constant $c$. In Table \ref{rates} we give for different values of $c$ the corresponding rates setting $\delta$ equal to zero, so that the rates  can be interpreted as upper bounds.

\begin{table}[ht]
\caption{Relation between $c$ and $\frac{\zeta^*(c)}{1+\zeta^*(c)}$}
\label{rates}
\begin{tabular}{ll}
  \hline \hline
    $c$ & rate \\
    \hline
    $0.0$ & $0.54$\\
    $0.1$ & $0.51$\\
    $0.2$ & $0.47$\\
    $0.3$ & $0.43$\\
    $0.5$ & $0.39$\\
    $0.6$ & $0.34$\\
    $0.7$ & $0.29$\\
    \hline \hline
\end{tabular}
 \end{table}
\end{remark}

\begin{remark}
It is also important to have an estimator for the variance of the error term $\sigma^2$, denoted by $\hat\sigma^2_{n,m}$. A consistent estimation of the variance is given by $||U^m||^2_{2,n}$ at the stopping time $m=m^*$.
\end{remark}

\subsection{Orthogonal $L_2$Boosting in a high-dimensional setting with bounded restricted eigenvalue assumptions}

In this section we analyze orthogonal $L_2$Boosting. For the orthogonal case, we obtain a faster rate of convergence than with the variant analyzed in the section before. We make use of similar notation as in the previous subsection: $U^m_o$ denotes the residual and $V^m_o$ the prediction error, formally defined below. Again, define $\beta^m_o$ as the parameter estimate after the $m^{th}$ iteration.

The orthogonal Boosting Algorithm was introduced in the previous section. For completeness we give here the full version with some additional notation which will be required in a later analysis.





\begin{algorithm}[Orthogonal $L_2$Boosting]
\begin{enumerate}
\item Initialization: Set $\beta^0_o=0$, $f^0_o=0$, $U^0_o=Y$ and the iteration index $m=0$.

\item Define $X^m_o$ as the matrix of all $X_{j_k}$, $k=0,1,\ldots,m$ and $P^m_o$ as the projection matrix given $X^m_o$.

\item Let $j_m$ be the maximizer of the following: $\max_{1\leq j\leq p}\rho^2(X_j,U^m_o)$.  Then, $f^{m+1}_o=P^m_o Y$ with corresponding regression projection coefficient $\beta^{m+1}_o$. 

\item Calculate the residual $U^m_o=Y-X\beta^m_o=(I-P^m_o) Y:=M_{P^m_o}Y$ and $V^m_o=M_{P^m_o}X\beta$.
\item Increase $m$ by one. If some stopping criterion is reached, then stop; else continue with Step 2.
\end{enumerate}
\end{algorithm}
It is easy to see that:
\begin{equation}\label{eq: ineq}
||U^{m+1}_o||_{2,n}\leq ||U^{m}_o||_{2,n}
\end{equation}

The benefit of the oBA method, compared to $L_2$Boosting, is that once a variable $X_j$ is selected, the procedure will never select this variable again. This means that every variable is selected at most once.

For any square matrix $W$, we denote by $\phi_s(W)$ and $\phi_l(W)$ the smallest and largest eigenvalues of $W$.





Denote by $T^m$ the set of variables selected at the $m^{th}$ iteration. Write $S^m:=T-T^m$. We know that $|T^m|= m$ by construction. Set $S_{c}^m=T^m-S^m$.

\begin{lemma}[lower bound of the remainder]\label{lemma:lowerboundoga}
Suppose assumptions A.1-A.3 hold. For any $m$ such that $|T^m|<s\eta$, $||V^m_o||_{2,n}\geq c_1 ||X\beta_{S^m}||_{2,n}$, for some constant $c_1>0$.
If $S^m=\emptyset$, then $||V^m_o||_{2,n}=0$.
\end{lemma}

The above lemma essentially says that if $S^m$ is non-empty, then there is still room for significant improvement in the value $||V^m||_{2,n}^2$. The next lemma is key and shows how rapidly the estimates decay. It is obvious that $||U^m_o||_{2,n}$ and $||V^m_o||_{2,n}$ are both decaying sequences. Before we state this lemma, we introduce an additional assumption.

\begin{assumption}
\label{betamin}
$\min_{j\in T}|\beta_j|\geq J$ and $\max_{j\in T}|\beta_j|\leq J'$ for some constants $J>0$ and $J'<\infty$.
\end{assumption}

\begin{remark}
The first part of assumption A.\ref{betamin} is known as a \textquotedblleft beta-min\textquotedblright assumption as it restricts the size of the non-zero coefficients. It can be relaxed so that the coefficients $\beta_j$ are a decreasing sequence in absolute value.
\end{remark}

\begin{lemma}[upper bound of the remainder]\label{OGA_roc}
Suppose assumptions A.1-A.4 hold. Assume that $\sqrt{s}\lambda_n\rightarrow 0$. Let $m^*$ be the first time that $||U^{m^*}||_{2,n}^2<\sigma^2+2K\sigma s\lambda_n^2$. Then, $m^*<Ks$ and $||V^m_o||_{2,n}^2\precsim {s\log(p)/n}$ with probability going to 1.
\end{lemma}

Although in general $L_2$Boosting may have a slower convergence rate than LASSO, oBA reaches the rate of LASSO (under some additional conditions). The same technique used in Lemma 6 also holds for the post-$L_2$Boosting case. Basically, we can prove, using similar arguments, that $T^{Ks}\supset T$ when $K$ is a large enough constant. Thus, post-$L_2$Boosting enjoys the LASSO convergence rate under assumptions A.1--A.4. We state this in the next section formally.

\subsection{Post-$L_2$Boosting with Componentwise Least Squares}
In many cases, penalization estimators like LASSO introduce some bias by shrinkage. The idea of \textquotedblleft post-estimators\textquotedblright\ is to estimate the final model by ordinary least squares including all the variables which were selected in the first step. We introduced post-$L_2$Boosting in Section 2. Now we give the convergence rate for this procedure. Surprisingly, it improves upon the rate of convergence of $L_2$Boosting and reaches the rate of LASSO (under stronger assumptions). The proof of the result follows the idea of the proof for Lemma \ref{OGA_roc}.


\begin{lemma}[Post-$L_2$Boosting]\label{lemma:post-boosting}
Suppose assumptions A.1-A.4 hold. Assume that $\sqrt{s}\lambda_n\rightarrow 0$. Let $m^*=Ks$ be the stopping time with $K$ a large enough constant.
Let $S^m=T^{m}-T^0$ be the set of variables selected at steps $1,2,....,m$. Then, $T^0\subset S^{m^*}$ with probability going to 1, i.e., all variables in $T^0$ have been revisited, and
$||P_{X_{T^{m^*}}}Y-X\beta||_{2,n}^2\leq CKs\lambda_n^2 \precsim {s\log(p)/n} $.
\end{lemma}

\begin{remark}
Our procedure is particularly sensitive to the starting value of the algorithm, at least in the non-asymptotic case. This might be exploited for screening and model selection: the procedure commences from different starting values until stopping. Then the intersection of all selected variables for each run is taken. This procedure might establish a sure screening property.
\end{remark}







\section{Simulation Study}
In this section we present the results of our simulation study. The goal of this exercise is to illustrate the relevance of our theoretical results for providing insights into the functionality of boosting and the practical aspects of boosting. In particular, we demonstrate that the stopping rules for early stopping we propose work reasonably well in the simulations and give guidance for practical applications. Moreover, the comparison with LASSO might also be of interest. First, we start with an illustrative example, later we present further results, in particular for different designs and settings.

\subsection{Illustrative Example}
The goal of this section is to give an illustration of the different stopping criteria. We employ the following data generating process (dgp):\footnote{In order to allow comparability the dgp is adopted from \cite{nr:buhlmann:2006}.}

\begin{equation}
y = 5 x_1 + 2 x_2 + 1 x_3 + 0 x_4 + \ldots + 0 x_{10} + \varepsilon,
\end{equation}

with $\varepsilon \sim N(0,2^2)$, $X=(X_1,\ldots, X_{10}) \sim N_10(0, I_{10})$, $I_{10}$ denoting the identity matrix of size $10$.
To evaluate the methods and in particular the stopping criteria we conduct an analysis of both in-sample and out-of-sample mean squared error (MSE). For the out-of-sample analysis we draw a new observation for evaluation and calculation of the MSE. For the in-sample analysis we also repeat the procedure and form the average over all repetitions. In both cases we employ $60$ repetitions. The sample size is $n=20$. Hence we have $20$ observations to estimate $10$ parameters.

The results are presented in Figures 5.1 and 5.2. Both show how MSE depends on the number of steps of the boosting algorithm. We see that MSE first decreases with more steps, reaches its minimum and then starts to increase again due to overfitting. In both graphs the solution of the $L_2$Boosting algorithm convergences to the OLS solution. We also indicate the MSE of LASSO estimators as horizontal lines (with cross-validated choice of the penalty parameter and data-driven choice of the penalization parameter.) In order to find a feasible stopping criterion we have to rely on the in-sample analysis. Figure \ref{ins_MSE} reveals that the stopping criterion we introduced in the sections before performs very well and even better than stopping based on a corrected  AIC values which has been proposed in the literature as stopping criterion for boosting. The average stopping steps of our criterion and the corrected AIC-based criterion (AICc) are presented by the vertical lines. On average our criterion stops earlier than the AICc based one. As our criterion performs better then the AICc we will not report AICc results in the following subsection. For the post-estimator similar patterns arise and are omitted.
\begin{figure}[htp]
\label{oos_MSE}
\centering{
\includegraphics[scale=0.35]{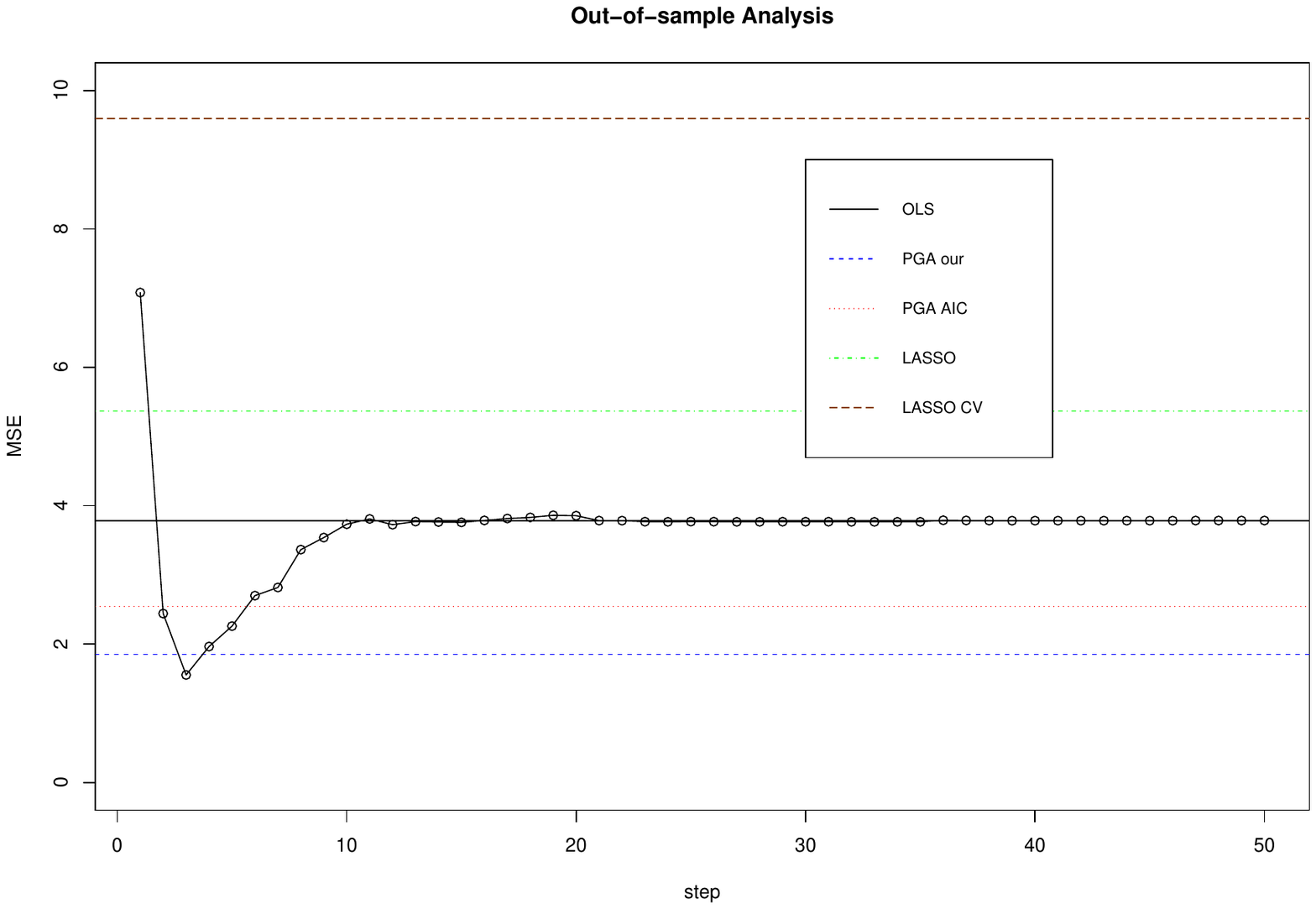}} 
\caption{Figure 5.1 shows the out-of-sample MSE of the $L_2$Boosting algorithm depending on the number of steps. The horizontal lines show the MSE of OLS and LASSO estimates.}
\end{figure}

\begin{figure}[htp]
\label{ins_MSE}
\centering{
\includegraphics[scale=0.35]{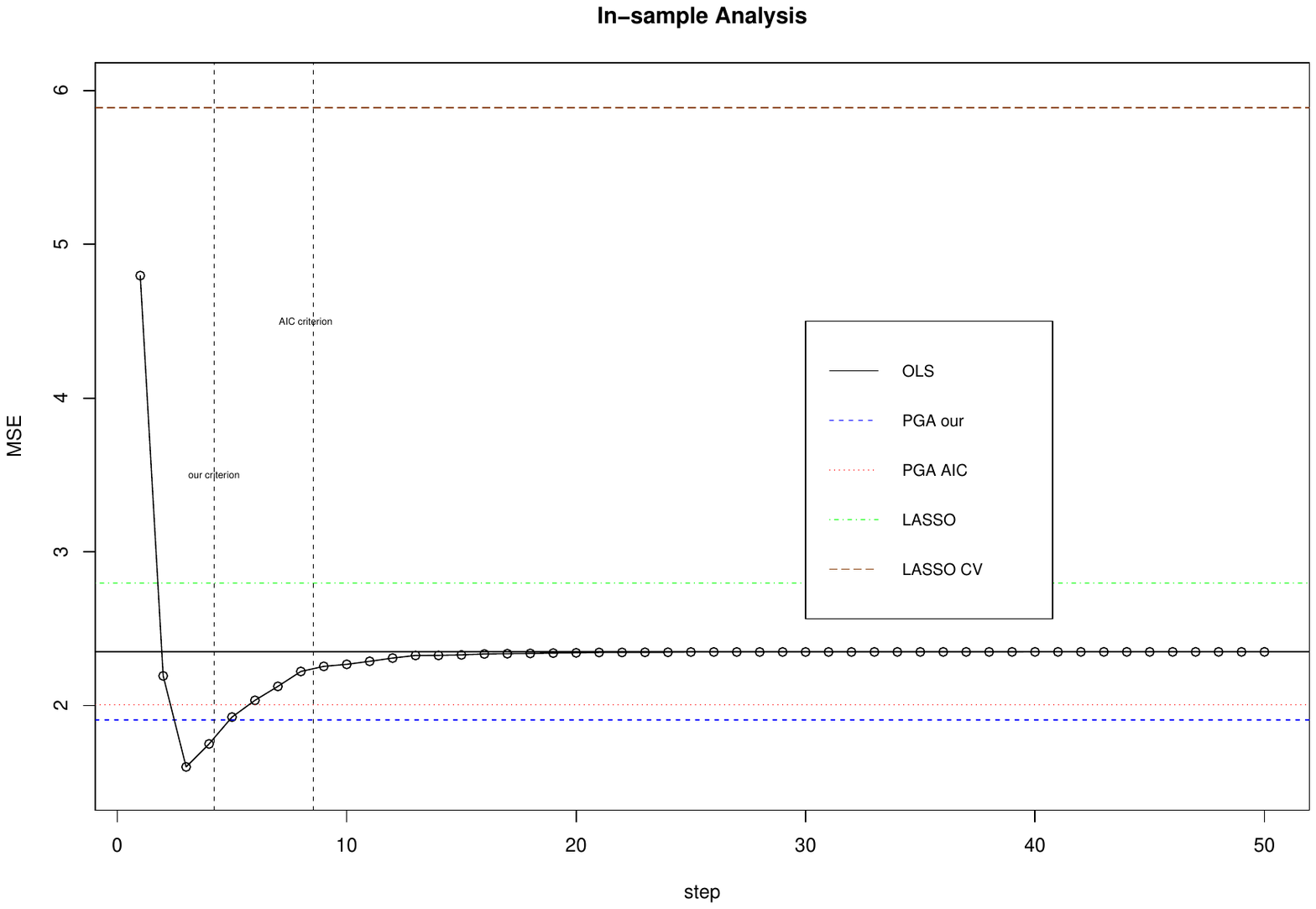}} 
\caption{Figure 5.2 shows the in-sample MSE of the $L_2$Boosting algorithm depending on the number of steps. The horizontal lines show the MSE of OLS and LASSO estimates.}
\end{figure}

\subsection{Further Results}
In this section we present results for different designs and settings so give a more detailed comparison of the methods.

We consider the linear model
\begin{equation}
y = \sum_{j=1}^p \beta_j x_j  + \varepsilon,
\end{equation}
with $\varepsilon$ standard normal distributed and i.i.d.
For the coefficient vector $\beta$ we consider two designs. First, we consider a sparse design, i.e., the first $s$ elements of $\beta$ are set equal to one, all other components to zero ($\beta=(1,\ldots,1,0,\ldots,0)$). Then we consider a polynomial design in which the $j$th coefficient given by $1/j$, i.e. $\beta=(1,1/2,1/3, \ldots, 1/p)$.

 For the design matrix $X$ we consider two different settings: an \textquotedblleft orthogonal\textquotedblright\ setting and a \textquotedblleft correlated\textquotedblright\ setting. In the former setting the entries of $X$ are drawn as i.i.d. draws from a standard normal distribution.
In the correlated design, the $x_i$ (rows of $X$) are distributed according to a multivariate normal distribution where the correlations are given by a Toeplitz matrix with factor $0.5$ and alternating signs.




We have the following settings:

\begin{itemize}
\item $X$: \textquotedblleft orthogonal\textquotedblright\ or \textquotedblleft correlated\textquotedblright
\item coefficient vector $\beta$: sparse design or polynomial decaying design
\item $n=100, 200, 400$
\item $p=100, 200$
\item $s=10$
\item $K=2$
\item out-of-sample prediction size $n_1=50$
\item number of repetitions $R=500$
\end{itemize}

We consider the following estimators: $L_2$Boosting with componentwise least squares, orthogonal $L_2$Boosting and LASSO. For Boosting and LASSO, we also consider the post-selection estimators (\textquotedblleft post\textquotedblright). For LASSO we consider a data-driven regressor-dependent choice for the penalization parameter (\cite{BCCH:2012}) and cross validation. Although cross validation is very popular, it does not rely on established theoretical results and therefore we prefer a comparison with the formal penalty choice developed in \cite{BCCH:2012}. For Boosting we consider three stopping rules: \textquotedblleft oracle\textquotedblright, \textquotedblleft Ks\textquotedblright, and a \textquotedblleft data-dependent\textquotedblright stopping criterion which stops if
$\frac{||U^m||_{2,n}^2}{||U^{m-1}||_{2,n}^2} = \frac{\hat{\sigma}_{m,n}^2}{\hat{\sigma}_{m-1,n}^2} > (1-C\log(p)/n))$ for some constant $C$. This means stopping, if the ratio of the estimated variances does not improve upon a certain amount any more.

The $Ks$-rule stops after $K \times s$ variables have been selected where $K$ is a constant. As $s$ is unknown, the rule is not directly applicable. The oracle rule stops when the mean-squared-error (MSE), defined below, is minimized, which is also not feasible in practical applications.

The simulations were performed in R (\cite{R:2014}). For LASSO estimation the packages \cite{hdm} and \cite{glmnet:2010} (for cross validation) were used. The Boosting procedures were implemented by the authors and the code is available upon request.

To evaluate the performance of the estimators we use the MSE criterion. We estimate the models on the same data sets and use the estimators to predict $50$ observations out-of-sample. The (out-of-sample) MSE is defined as
\begin{equation}
MSE = \mathbb{E} [(f(X)- f^m(X))^2]= \mathbb{E} [( X'(\beta - \beta^m))^2],
\end{equation}
where $m$ denotes the iteration at which we stop, depending on the employed stopping rule.
The MSE is estimated by
\begin{equation}
\hat{MSE} =  1/n_1 \sum_{i}^{n_1} [(f(x_i)- f^m(x_i))^2]= 1/n_1\sum_{i}^{n_1} [( x_i'(\beta - \beta^m))^2]
\label{eq:}
\end{equation}
for the out-of-sample predictions.

The results of the simulation study are shown in Tables 3 -- 10.

\input{Simulation_results_edit.tex}

As expected, the oracle-based estimator dominates in all cases except in the correlated, sparse setting with more parameters than observations. Our stopping criterion gives very good results, on par with the infeasible $Ks$-rule. Not surprisingly, given our results, both post-Boosting and orthogonal Boosting outperform the standard $L_2$Boosting in most cases. A comparison of post- and orthogonal Boosting does not provide a clear answer with advantages on both sides. It is interesting to see that the post-LASSO increases upon LASSO, but there are some exceptions, probably driven by overfitting. Cross validation works very well in many constellations. An important point of the simulation study is to compare Boosting and LASSO. It seems that in the polynomial decaying setting, Boosting (orthogonal Boosting with our stopping rule) dominates post-LASSO. This also seems true in the iid, sparse setting. In the correlated, sparse setting they are on par. Summing up, it seems that Boosting is a serious contender for LASSO.


\begin{remark}
It seems that in very high-dimensional settings, i.e. when the number of signals $s$ is bigger than the sample size $n$, (e.g. $n=20$, $p=50$, $s=30$) boosting performs quite well and outperforms LASSO which seems to break down. This case is not covered by our setting, but it is an interesting topic for future research and shows one of the advantages of boosting.
\end{remark}

\section{Applications}
In this section we present applications from different fields to illustrate the boosting algorithms. We present applications to demonstrate how the methods work when applied to real data sets and, then compare these methods to related methods, i.e. LASSO. The focus is on making predictions which is an important task in many instances.

\subsection{Application: Riboflavin production}
 Thos application  involves genetic data and analyzes the production of riboflavin. First, we describe the data set, then we present the results.

\subsubsection{Data set}
The data set has been provided by DSM (Kaiserburg, Switzerland) and was made publicly available for academic research in \cite{BKM:2014} (Supplemental Material).

The real-valued response / dependent variable is the logarithm of the riboflavin production rate. The (co-)variables measure the logarithm of the expression level of $4,088$ genes ($p=4,088$), which are normalized. This means that the covariables are standardized to have variance $1$, and the dependent variable and the resources are \textquotedblleft de-meaned\textquotedblright ,which is equivalent to including an unpenalized intercept. The data set consists of $n=71$ observations which were hybridized repeatedly during a fed-batch fermentation process in which different engineered strains and strains grown under different fermentation conditions were analyzed. For further details we refer to \cite{BKM:2014}, their Supplemental Material, and the references therein.

\subsubsection{Results}
We analyzed a data set about the production of riboflavin (vitamin $B_2$) with B.~subtilis.
We split the data set randomly into two samples: a training set and a testing set. We estimated models with different methods on the training set and then used the testing set to calculate out-of-sample mean squared errors (MSE) in order to evaluate the predictive accuracy.
The size of the training set was $60$ and the remaining $11$ observations were used for forecasting. The table below shows the MSE errors for different methods discussed in the previous sections.

\begin{table}[ht!]
\centering
\tiny
\caption{Results Riboflavin Production -- out-of-sample MSE}
\begin{tabular}{rrrrrrrr}
\hline\hline
            BA-Ks &  BA-our & p-BA-Ks & p-BA-our &  oBA-Ks  & oBA-our  &LASSO &p-LASSO \\ \hline
 $0.4669$& $0.3641$ & $0.4385$ & $0.1237$ & $0.4246$ & $0.1080$ & $0.1687$ & $0.1539$  \\
\hline\hline
\end{tabular}
\end{table}

All calculations were peformed in R (\cite{R:2014}) with the package \cite{hdm} and our own code. Replication files are available upon request.

The results show that, again, post- and orthogonal $L_2$Boosting give comparable results. They both outperform LASSO and post-LASSO in this application.

\subsection{Predicting Test Scores}
\subsubsection{Data Set}
Here the task is to predict the final score in the subjects Mathematics and Portugese in secondary education. This is relevant, e.g., to identify students which need additional support to master the material. The data contains both student grades and  demographic, social and school related features and it was collected by using school reports and questionnaires. Two datasets are provided regarding the performance in two distinct subjects: Mathematics and Portuguese. The data set is made available at the UCI Machine Learning Repository and was contributed by Paulo Cortez. The main reference for the data set is \cite{corsil2008}.

\subsubsection{Results}
We employed five-fold CV to evaluate the predictive performance of the data set. The results remain stable when choosing a different number of folds. The data sets contain, for both test results, $33$ variables, which are used as predictors. The data set for the Mathematics test scores contains  $395$  observations, the sample size for Portuguese is $649$. The results confirm our theoretical derivations that boosting is comparable to Lasso.

\begin{table}[ht!]
\centering
\tiny
\label{Studentperf}
\caption{Prediction of education -- out-of-sample MSE}
\begin{tabular}{rrrrrr}
\hline\hline
     subject   &    BA &  post-BA & oBA & Lasso &  post-Lasso \\ \hline
Mathematics & $19.1$ & $19.3$  & $19.3$    & $18.4$  & $18.4$ \\
Protugese &  $8.0$ &$7.9$  &$7.9$ &$7.8$ &$7.8$\\

\hline\hline
\end{tabular}
\end{table}

\section{Conclusion}

Although boosting algorithms are widely used in industry, the analysis of their properties in high-dimensional settings has been quite challenging. In this paper the rate of convergence for the $L_2$Boosting algorithm and variants are derived, which has been a long-standing open problem until now.

\appendix
\section{ A new approximation theory for PGA}

\subsection{New results on approximation theory of PGA}
\label{Appendix_Results_Rev}
In this section of the appendix we introduce preparatory results for a new approximation theory based on revisiting. These results are useful to prove Lemma \ref{Lemma:FR}. The proofs of these lemmas are provided in the next section.

For any $m_1\geq m$, define $L(m,m_1)=||V^{m_1}||_{2,n}^2/||V^{m}||_{2,n}^2\leq 1$. For any integers $q_1>q$, define $\Delta(q,q_1):=\Pi_{j=0}^{q_1-q-1}(1-\frac{1-c}{q+j})$ with some constant $c$. It is easy to see that for any $k_1>k$, $\Delta(k,k_1)/(k/k_1)^{1-c}>1$ and $\Delta(k,k_1)/(k/k_1)^{1-c}\rightarrow 1$ as $k\rightarrow \infty$.

First of all, we can establish the following naive bounds on $L(m,m_1)$.

\begin{lemma}\label{Lemma_naive_bound}
Suppose $||V^m||>0$, $m+1,m_1<M_0$

(a) For any $m$, $L(m,m+1)\leq 1-\frac{1-c}{q(m)}$.

(b) For any $m\geq 0$, $m_1> m$,
$L(m,m_1)\leq \Delta(q(m),q(m)+m_1-m)$.
\end{lemma}

The bound of $L(m,m_1)$ established in Lemma \ref{Lemma_naive_bound} is loose. To obtain better results on the convergence rate of $||V^m||^2$, the revisiting behavior of the PGA has to be analyzed in more detail. The  revisiting behavior of PGA addresses the question when and how often  variables are selected again which have been already selected before. When PGA chooses too many new variables, it leads on average to slower convergence rates and vice versa. The next results primarily focus on analyzing the revisiting behavior of the PGA.

The following lemma summarizes a few basic facts of the sequence of $A_i$, $i\geq 1$.

\begin{lemma}\label{Lemma:simple results on revisiting}
Suppose $m,m_1<M_0$. Suppose further that conditions A.1 and A.2 are satisfied.

(1) If $E_n[X_i'X_i]$ is diagonal matrix, i.e., $c=0$, then there are only $R$s in the sequence $\mathcal{A}$.

(2) Define $N(m):=\{k|A_k=N,1\leq k\leq m\}$, the index set for the non-revisiting steps, and $R(m):=\{k|A_k=R,1\leq k\leq m\}$, the index set for the revisiting steps. Then $|R(m)|+|N(m)|=m$, $q(m)=|N(m)|+q(0)$, and $J_N(m):=\{j^k|k\in N(m)\}$ has cardinality equal to $|N(m)|$.

(3) $L(0,m)\leq \Pi_{i=1}^{|N(m)|}(1-\frac{1-c}{q(0)+i-1})\times (1-\frac{1-c}{q(m)})^{|R(m)|}$, i.e., the sequence to maximize the upper bound of $L(0,m)$ stated above is $NN...NRR...R$. Consequently, the sequence $\{A_{m+1},...,A_{m_1}\}$ to maximize the upper bound of $L(m,m_1)$ for general $m_1>m$ is also $NN...NRR...R$.
\end{lemma}

The proof of this lemma is obvious and hence omitted. Much more involved are the following results, for characterizing the revisiting behavior.

\begin{lemma}
\label{Lemma:naive bound1}
Suppose conditions A.1 and A.2 are satisfied. Suppose further that there is a consecutive subsequence of $N$s in the sequence $\mathcal{A}$ starting at position $m$ with length $k$. Assume that $m+k<M_0$.Then for any $\delta>0$, there exists an absolute constant $Q(\delta)>0$ such that for any $q(m)>Q(\delta)$, the length of such a sequence cannot be longer than $\left(((1+\delta)\frac{(2-c)(1+c)}{(2+c)(1-c)})^{\frac{1}{1-c}}-1\right)q(m)$.
\end{lemma}

Lemma \ref{Lemma:naive bound1} establishes some properties of the $N$ sequence. Next, we formulate a lemma which characterizes a lower bound of the proportions of $R$s.

\begin{lemma}\label{Lemma:R}
Assume that assumptions A.1--A.2 hold. Assume that $m<M_0$. Consider the sequence of steps $1,2,...,m$. Set $\mu_e(c)=(1-\exp(-1/(1-c)^2))$ for any $c\in (0,1)$.
Then, the number of $R$s in the sequence $\mathcal{A}$ satisfies:
$$|R(m)|\geq \frac{1-\mu_e(c)}{2-\mu_e(c)}m-\frac{\mu_e(c)}{2-\mu_e(c)}q(0).$$
\end{lemma}

Lemma \ref{Lemma:R} illustrates that for any $1>c>0$, the $R$ spots occupy at least some significant proportion of the sequence $\mathcal{A}$, with the lower bound of the proportion depending on $c$. In fact, such a result holds for arbitrary consecutive sequence $A_m,A_{m+1},...,A_{m+k}$, as long as $m+k<M_0$. In the main text, we further extend results stated in Lemma \ref{Lemma:R}.

\subsection{Proofs of Lemmas in Appendix \ref{Appendix_Results_Rev}}

\begin{proof}[Proof of Lemma \ref{Lemma_naive_bound}]

By definition, $||V^m||^2=\sum_{j\in T^m} \alpha_j^m <V^m,X_j>= \sum_{j\in T^m} \alpha_j^m||V^m||corr(V^m,X_j)$. Define $\rho_{j^m}:=|\gamma^m_{j^m}|/||V^m||=|corr(V^m,X_{j^m})|$. Therefore, $\rho_{j^m}|\sum_{j\in T^m}\alpha^m_j|\geq ||V^m||$, i.e., $\rho_{j^m}^2 |\sum_{j\in T^m}\alpha^m_j|^2\geq ||V^m||^2$.

By the Cauchy-Schwarz inequality, $|\sum_{j\in T^m}\alpha^m_j|^2\leq q(m)||\alpha^m||^2$. Therefore, $\rho_{j^m}^2\geq \frac{1-c}{q(m)}$.

So $||V^{m+1}||_{2,n}^2=||V^m||_{2,n}^2(1-\rho_{j^m}^2)\geq ||V^m||^2(1-\frac{1-c}{q(m)})$, i.e., $L(m,m+1)\leq 1-\frac{1-c}{q(m)}$. The second statement follows from statement (a) and the fact that $q(m'+1)\leq q(m')+1$ for any $m'\geq 0$.
\end{proof}

\begin{proof}[Proof of Lemma \ref{Lemma:naive bound1}]

Denote the length of such a sequence of $N$s as $k$. Then by definition of non-revisiting, $||\alpha^{m+k}||^2=||\alpha^m||^2+\sum_{j=m}^{m+k-1}(\gamma^j)^2$.

We know that $||V^m||_{2,n}^2=||V^{m+k}||_{2,n}^2+\sum_{j=m}^{m+k-1}(\gamma^j)^2,$ therefore,
$$||V^m||_{2,n}^2=||V^{m+k}||_{2,n}^2+||\alpha^{m+k}||^2-||\alpha^m||^2.$$
Applying Assumption A.2, $(2+c)/(1+c)||V^m||_{2,n}^2\leq (2-c)/(1-c)||V^{m+k}||_{2,n}^2$. Consequently, $L(m,m+k)\geq \frac{(2+c)(1-c)}{(2-c)(1+c)}$, with the right-hand side of the inequality being a constant that only depends on $c$. When $c=0$, the constant equals $1$, which implies that $k$ has to be $0$, i.e., there are no $N$s in the sequence $\mathcal{A}$.

For any $\delta>0$, there exists $Q>0$ such that for any $q(m)>Q$, $L(m,m+k)\leq \Delta(q(m),q(m)+k)\leq (1+\delta)(\frac{q(m)}{q(m)+k})^{1-c}$. It follows that $(\frac{q(m)}{q(m)+k})^{1-c}\geq \frac{1}{1+\delta}\frac{(2+c)(1-c)}{(2-c)(1+c)}$, i.e., $k\leq \left(((1+\delta)\frac{(2-c)(1+c)}{(2+c)(1-c)})^{\frac{1}{1-c}}-1\right)q(m)$.
\end{proof}

\begin{proof}[Proof of Lemma \ref{Lemma:R}]

Define $\widetilde{N}(m)$ as: $$\{l:j^l\notin T^0,j^l \textrm{ is only visited once within steps 1,2,...,m}\}.$$ So it is easy to see that $\widetilde{N}(m)\subset N(m)$ and $|\widetilde{N}(m)|\geq 2|N(m)|-m$. Therefore, for any $j^l$ with $l\in \widetilde{N}(m)$, $\alpha^{m}_{j^l}=-\gamma^l$.

If $|R(m)|\geq m/2$, then we already have the results stated in this lemma. Otherwise, $\widetilde{N}(m)$ is non-empty. Therefore, $||\alpha^{m}||^2\geq \sum_{l\in \widetilde{N}(m)} (\gamma^{l-1})^2$. By the sparse eigenvalue condition A.2,
\begin{equation}\label{eq:R bound}
\frac{1}{1-c}||V^{m}||_{2,n}^2\geq||\alpha^{m}||^2\geq \sum_{l\in \widetilde{N}(m)} (\gamma^{l-1})^2.
\end{equation}

Note that by Lemma \ref{Lemma:naive bound1}, $(\gamma^{l-1})^2=||V^{l-1}||_{2,n}^2-||V^l||\geq \frac{1-c}{q(l-1)}||V^{l-1}||^2$. Therefore, $(\gamma^{l-1})^2\geq \frac{1-c}{q(l-1)}||V^m||_{2,n}^2$. Plugging the above inequality back into (\ref{eq:R bound}), we get:

$$\frac{1}{1-c}||V^{m}||_{2,n}^2\geq (1-c)\sum_{l\in \widetilde{N}(m)}\frac{1}{q(l-1)}||V^m||^2.$$

Since these $q(l-1)$, $l\in \widetilde{N}(m)$, are different integers with their maximum being less than or equal to $q(m)=q(0)+|N(m)|$. Therefore, $\sum_{l\in \widetilde{N}(m)} \frac{1}{q(l-1)}\geq \sum_{l=1}^{|\widetilde{N}(m)|}\frac{1}{q(0)+|N(m)|-l}\geq \log((q(0)+|N(m)|)/(q(0)+|N(m)|-|\widetilde{N}(m)|))$.

The above inequality implies that $\exp(1/(1-c)^2)\geq (q(0)+|N(m)|)/(q(0)+|N(m)|-|\widetilde{N}(m)|)$, i.e., $|\widetilde{N}(m)|\leq (1-\exp(-1/(1-c)^2))(q(0)+|N(m)|)$.

Set $\mu_e(c)=(1-\exp(-1/(1-c)^2))\in (0,1)$ when $c\in [0,1)$.

Since we know that $|\widetilde{N}(m)|\geq 2|N(m)|-m$, we immediately have: $|N(m)|\leq \frac{1}{2-\mu_e(c)}(m+\mu_e(c)q(0))$, and $|R(m)|\geq \frac{1-\mu_e(c)}{2-\mu_e(c)}m-\frac{\mu_e(c)}{2-\mu_e(c)}q(0)$.
\end{proof}

\section{Proofs for Section 3}

\begin{proof}[Proof of Lemma \ref{Lemma:FR}]

First of all, WLOG, we can assume that $q(0)$ exceeds a large enough constant $Q(\delta)$. Otherwise, we can consider the true parameter $\beta$ contains some infinitesimal components such that $q(0)>Q(\delta)$.

Let's revisit inequality (\ref{eq:R bound}). $\sum_{l\in \widetilde{N}(m)} (\gamma^{l-1})^2\geq\sum_{l\in \widetilde{N}(m)}||V^{l-1}||^2\frac{1-c}{q(l-1)}$.
The right-hand side reaches its minimum when $\widetilde{N}(m)=\{m-|\widetilde{N}(m)|+1,m-|\widetilde{N}(m)|+2,...,m\}$, and for the step $m-|\widetilde{N}(m)|+l$, $q(m-|\widetilde{N}(m)|+l-1)=q(m)-|\widetilde{N}(m)|+l-1$, $l=1,2,...,|\widetilde{N}(m)|$. We know that $||V^{l-1}||_{2,n}^2/||V^m||_{2,n}^2\geq 1/L(l-1,m)$, while $L(m-l,m)\rightarrow (\frac{q(m)-l-1}{q(m)-1})^{1-c}$ as $q(m)-m+l\rightarrow \infty$. So for any $\delta>0$, and $q(0)$ large enough, $(1+\delta)\sum_{l\in \widetilde{N}(m)}||V^{l-1}||_{2,n}^2\frac{1-c}{q(l-1)}\geq ||V^m||_{2,n}^2(1-c)\sum_{l=1}^{|\widetilde{N}(m)|}\frac{1}{q(m)-l}\times (\frac{q(m)-1}{q(m)-l-1})^{1-c}\geq \frac{1}{1-c}q(m)^{1-c}((q(m)-|\widetilde{N}(m)|)^{c-1}-q(m)^{c-1})$.

Combining the above inequality with (\ref{eq:R bound}), we get:
$\frac{1+\delta}{1-c}\geq (1-c)q(m)^{1-c}((q(m)-|\widetilde{N}(m)|)^{c-1}-q(m)^{c-1})$, i.e, $|\widetilde{N}(m)|\leq q(m)[1-(1+\frac{1+\delta}{(1-c)^2})^\frac{-1}{1-c}]\leq (1+\delta')\mu_a(c)q(m)$, for some $\delta'>0$, with $\delta'\rightarrow 0$ as $\delta\rightarrow 0$. The rest of the arguments follow the proof stated for Lemma \ref{Lemma:R}.

Hence, the results stated in Lemma \ref{Lemma:FR} hold.
\end{proof}

\begin{proof}[Proof of Lemma \ref{lemma:appr-ori}]

Without loss of generality, we can assume that $k=0$. We can also assume that $||V^0||_{2,n}^2>0$, because otherwise $||V^0||_{2,n}^2=||V^{m}||_{2,n}^2=0$ so that the conclusion already holds. Set $n_0=|N(m)|\leq (1+\delta)\frac{m+\mu_a(c)q(0)}{2-\mu_a(c)}$ for some $\delta>0,$ when $q(0)$ is large enough.

Then, it easy to see that $||V^{m}||_{2,n}^2/||V^0||_{2,n}^2\leq \Pi_{i=1}^{n_0}(1-\frac{1-c}{q(0)+i-1}) (1-\frac{1-c}{q(0)+n_0})^{(m-n_0)}$, and the right hand reaches its maximum, when $n_0=(1+\delta)n^*_0$ with $n^*_0:=\frac{m+\mu_a(c)q(0)}{2-\mu_a(c)}$. When $q(0)$ is large enough, we know that there exists a $\delta>0$ such that $$\Pi_{i=1}^{n^*_0}(1-\frac{1-c}{q(0)+i-1})\leq (1+\delta)(\frac{q(0)}{q(0)+n^*_0})^{1-c}=(1+\delta)(\frac{2-\mu_a(c)}{2+\lambda})^{1-c}$$ and 
\begin{align*}
(1-\frac{1-c}{q(0)+n^*_0})^{m-n^*_0}&\leq (1+\delta)(1-\frac{1-c}{q(0)\frac{\lambda+2}{2-\mu_a(c)}})^{q(0)\frac{(1-\mu_a(c))\lambda-\mu_a(c)}{2-\mu_a(c)}}\\
&\leq (1+\delta)\exp(-\frac{(1-c)((1-\mu_a(c))\lambda-\mu_a(c))}{2+\lambda}).
\end{align*}
Thus, for any $\delta>0$, and for $q(0)$ large enough,
\[||V^{m}||_{2,n}^2/||V^0||_{2,n}^2\leq (1+\delta)(\frac{2-\mu_a(c)}{2+\lambda})^{1-c}\exp(-\frac{(1-c)((1-\mu_a(c))\lambda-\mu_a(c))}{2+\lambda}). \] Notice that the bound on the right-hand side does not depend on $q(0)$ or $m$.

As defined in the statement of this lemma, $\zeta(c,\lambda)=\frac{\frac{(1-c)((1-\mu_a(c))\lambda-\mu_a(c))}{2+\lambda}}{\log(\frac{2+\lambda}{2-\mu_a(c)})}+1-c$, where $\frac{2+\lambda}{2-\mu_a(c)}=\frac{q(0)+n_0^*}{q(0)}$ by definition.

So for any $\delta>0$, and for $q(0)$ large enough, $$||V^{m}||_{2,n}^2\leq ||V^0||_{2,n}^2 (\frac{q(0)}{q(0)+n^*_0})^{\zeta(c,\lambda)-\delta}.$$
\end{proof}

\begin{proof}[Proof of theorem \ref{theorem:appr}]

If $q(0)<Q(\delta)$ where $Q(\delta)$ is defined in Lemma \ref{Lemma:FR}, we can treat $\beta$ as if there are additional infitestimony coefficients so that $q(0)=Q(\delta)$.

Let $\lambda^*$ be the maximizer of $\zeta(c,\lambda)$ given $c\in (0,1)$.
For any small $\delta>0$, define a sequence $m_0,m_1,\ldots$ according to the following rule:
$$m_0=s,m_{i+1}=\lceil m_i+\lambda^*n_i \rceil$$, $i=1,2,\ldots,$ with the sequence $n_0,n_1,\ldots$ being defined as:
$$n_{i+1}=\lfloor n_i+(1+\delta)\frac{1}{2-\mu_a(c)}(m_{i+1}-m_i+\mu_a(c)n_i)\rfloor,$$
with $n_0=s$.

It is easy to see that:
By Lemma \ref{Lemma:FR},

(1). $1<c_{\lambda^*}<m_{i+1}/m_i\leq C_{\lambda^*}$, for some constant $c_{\lambda^*},C_{\lambda^*}$ that only depends on $\lambda^*$ and $i\geq I(\delta)$, where $I(\delta)$ is a fixed real number depending on $\delta$.

(2). $c_n\leq n_i/m_i\leq C_n$, for $i\geq I(\delta)$, with $c_n, C_n$ being generic constants.

(3). $n_i\geq q(m_i)$, for $i\geq I(\delta)$.

And by Lemma \ref{lemma:appr-ori},

(4).  \begin{align*}
    ||V^{m_{i+1}}||_{2,n}^2/||V^{m_{i}}||_{2,n}^2&\leq (\frac{q(m_{i})}{q(m_i)+\frac{1}{2-\mu_a(c)}(m_{i+1}-m_i+\mu_a(c)q(m_i))})^{\zeta^*(c)-\delta}\\
    &\leq (\frac{n_i}{n_i+\frac{1}{2-\mu_a(c)}(m_{i+1}-m_i+\mu_a(c)n_i)})^{\zeta^*(c)-\delta}\\
    &\leq (\frac{n_i}{n_{i+1}})^{\zeta^*(c)-\delta},
    \end{align*} for all $i\geq I(\delta)$.
So, according to statements 1--4, we are able to conclude that:
\begin{equation}\label{eq:main_app_bound}
||V^{m_i}||_{2,n}^2\precsim C||V^0||_{2,n}^2 (\frac{s}{n_i})^{\zeta^*(c)-\delta}\precsim C||V^0||_{2,n}^2 (\frac{s}{m_i+s})^{\zeta^*(c)-\delta},
\end{equation}
for all $i\geq I(\delta)$, with $C$ being a constant.

For any $m>0$, $m<M_0$, since $m_0,m_1,...$ is an increasing sequence of positive integers, there exists $i$ such that $m_i\leq m<m_{i+1}$. So $\frac{m}{m_i}\leq \frac{m_{i+1}}{m_i}\leq C_{\lambda^*}$. Also, for $m$ large enough, $i$ must be sufficiently large that $i\geq Q(\delta)$.
Therefore, $||V^m||_{2,n}^2\leq ||V^{m_i}||_{2,n}^2\precsim ||V^0||_{2,n}^2 (\frac{s}{s+m_i})^{\zeta^*(c)-\delta}\precsim ||V^0||_{2,n}^2(\frac{s}{s+m})^{\zeta^*(c)-\delta}$.
\end{proof}

\section{Proofs for Section 4}

The two lemmas below state several basic properties of the $L_2$Boosting algorithm that will be useful in deriving the main results.

\begin{lemma}\label{lemma:uv}
$||U^{m+1}||_{2,n}^2=||U^m||_{2,n}^2-<U^m,X_{j_m}>_{n}^2=||U^m||_{2,n}^2(1-\rho^2(U^m,X_{j_m}))$, and
$||V^{m+1}||_{2,n}^2=||V^m||_{2,n}^2-2<V^m,\gamma^m_{j_m}X_{j_m}>_n+(\gamma^m_{j_m})^2,$
where $\gamma^m_{j_m}=<U^m,X_{j_m}>_n$.

Moreover, since $V^m=U^m-\varepsilon$,
$||V^{m+1}||_{2,n}^2=||V^m||_{2,n}^2-2<U^m,X_{j_m}>_n<\varepsilon,X_{j_m}>_n-<U^m,X_{j_m}>_n^2=||V^m||_{2,n}^2-2\gamma^m_{j_m}<\varepsilon,X_{j_m}>_n-(\gamma^m_{j_m})^2.$
\end{lemma}

\begin{lemma}\label{lemma: Zm}
Assuming that assumptions A.1-A.3 hold, and $m\leq M_0$. Let $Z_m=||U^m||^2_{2,n}-||V^m||^2_{2,n}$. Then, with probability $\geq 1-\alpha$ and uniformly in $m$, $|Z_m-\sigma_n^2|\leq \frac{2\sqrt{m+s}}{\sqrt{1-c}}\lambda_n ||V^m||_{2,n}$.
\end{lemma}

Lemma \ref{lemma: Zm}  bounds the difference between $||U^m||^2_{2,n}$ and $||V^m||^2_{2,n}$. This difference is $\sigma^2(1-O_p(s/n))$ if $\beta^m=\beta$.

\begin{proof}[Proof of Lemma \ref{lemma: Zm}]

From Lemma 1, $Z_{m+1}=Z_m-2\gamma_{j^m}^m<\varepsilon,X_{j^m}>_n$, and
$Z_{m}=Z_0-2\sum_{k=0}^{m-1}\gamma_{j^k}^k <\varepsilon ,X_{j^k}>=Z_0-2<\varepsilon,X\beta-V^m>$.
$Z_0=||y||_{2,n}^2-||X\beta||_{2,n}^2=||\varepsilon||_{2,n}^2+2<\varepsilon,X\beta>$.

Then,
\begin{align*}
    Z_m&=||\varepsilon||_{2,n}^2+2<\varepsilon,V^m>\\
    &=||\varepsilon||_{2,n}^2+2<\varepsilon,\frac{V^m}{||V^m||_{2,n}}>||V^m||_{2,n}\leq ||V^m||_{2,n}\lambda_n (||\alpha^m||_1/||V^m||)\\
    &\leq ||V^m||_{2,n}\sqrt{m+s}\lambda_n(||\alpha^m||/||V^m||)
\end{align*}
since $|supp(V^m)|\leq m+s$. By assumption A.2, $||\alpha^m||/||V^m||\leq \frac{1}{1-c}$. Hence, the conclusion holds.
\end{proof}

\subsection{Proofs for $L_2$Boosting}

\begin{proof}[Proof of Lemma \ref{Lemma:Bounds}]

We assume that $\lambda_n\geq \max_{1\leq j\leq p} |<\epsilon,X_j>_n|$. This event occurs with probability $\geq 1-\alpha$.

According to our definition, ${m}^*+1$ is the first time $||V^m||_{2,n}\leq \eta\sqrt{m+s}\lambda_n$, where $\eta$ is a positive constant. We know that in high-dimensional settings, $||U^m||_{2,n}\rightarrow 0$, so $||V^m||_{2,n}\rightarrow \sigma^2$. Thus, by fixing $p$ and $n$, such an $m^*$ must exist.

First, we prove that for any $m<\widetilde{m}:=(m^*+1)\wedge M_0$, we have: $||V^{m+1}||_{2,n}^2\leq ||V^m||_{2,n}^2$, i.e., $||V^m||_{2,n}^2$ is non-increasing with $m$.

By Lemma 1, $||V^{m+1}||_{2,n}^2=||V^m||_{2,n}^2-\gamma^m(\gamma^m-2<\varepsilon,x_{j^m}>_n)$.

To show that $||V^m||_{2,n}^2$ is non-increasing with $m$, we only need to prove that $\gamma^m$ and $(\gamma^m-2<\varepsilon,x_{j^m}>_n)$ have the same sign, i.e., $|\gamma^m|>2|<\varepsilon,x_j>_n|$. It suffices to prove $|\gamma^m|>2\lambda_n$. We know that $|\gamma^m|\geq \sqrt{(1-c)}\frac{||V^m||_{2,n}}{\sqrt{m+s}}-\lambda_n\geq \lambda_n(\eta\sqrt{1-c}-1)$. Thus, for any $\eta>\frac{3}{\sqrt{1-c}}$, $||V^{m+1}||_{2,n}^2\leq ||V^m||_{2,n}^2$ for all $m<\widetilde{m}$.

Define $q(m)$ as in Section 3.1, with $q(0)=s$.

For any $m<M_0\wedge (m^*+1)$, by selecting a variable that is the most correlated with $V^m$, we are able to reduce $||V^m||_{2,n}^2$ by at least $\frac{1-c}{q(m)}||V^m||_{2,n}^2$, and thus $||U^m||_{2,n}^2-||U^{m+1}||_{2,n}^2= ||V^m||_{2,n}^2-||V^{m+1}||_{2,n}^2-2\gamma^m<X_{j^m},\epsilon>_n\geq (\gamma^m)^2-2\lambda_n|\gamma^m|$.

Define $\widetilde{\gamma}^m:=\frac{\sqrt{1-c}}{\sqrt{q(m)}}||V^m||_{2,n}-\lambda_n$. Consider the variable $j'$ that is most correlated with $V^m$, and define $\gamma'=<X_{j'},V^m>_n$. By Lemma 1, $|\gamma'|\geq \frac{\sqrt{1-c}}{\sqrt{q(m)}}||V^m||_{2,n}$. Consequently, $|<X_{j'},U^m>_n|=|\gamma'+<X_{j'},\epsilon>_n|\geq \widetilde{\gamma}^m$.

By definition, $|\gamma^m|\geq |<X_{j'},U^m>_n|\geq \widetilde{\gamma}^m$.

Since we assume that $||V^m||_{2,n}>\eta\sqrt{m+s}\lambda_n$, so $\widetilde{\gamma}^m:=\frac{\sqrt{1-c}}{\sqrt{q(m)}}||V^m||_{2,n}-\lambda_n>\lambda_n$. Therefore, $|\gamma^m|> \lambda_n$, and $(\gamma^m)^2-2\lambda_n|\gamma^m|\geq (\widetilde{\gamma}^m)^2-2\lambda_n\widetilde{\gamma}^m$.


By Lemma \ref{lemma:uv}, $||V^m||_{2,n}^2-||V^{m+1}||_{2,n}^2= ||U^m||_{2,n}^2-||U^{m+1}||_{2,n}-2\gamma^m<X_{j^m},\epsilon>_n\geq |\gamma^m|^2-2\lambda_n|\gamma^m|\geq |\widetilde{\gamma}^m|^2-2\lambda_n \widetilde{\gamma}^m$

$=\frac{1-c}{q(m)}||V^m||_{2,n}^2-4\frac{\sqrt{1-c}}{\sqrt{q(m)}}\lambda_n ||V^m||_{2,n}+3\lambda_n^2\geq \frac{1-c}{q(m)}||V^m||_{2,n}^2-4\frac{\sqrt{1-c}}{\sqrt{q(m)}}\lambda_n ||V^m||_{2,n}.$

Thus,
\begin{equation}\label{eq:Vm}
||V^{m}||_{2,n}^2-||V^{m+1}||_{2,n}^2\geq \frac{1-c}{q(m)}||V^m||_{2,n}^2-4\frac{\sqrt{1-c}}{\sqrt{q(m)}}\lambda_n ||V^m||_{2,n}.
\end{equation}

Plugging in $||V^k||_{2,n}>\eta\sqrt{k+s}\lambda_n$ to inequality (\ref{eq:Vm}), for any $k>0,k<M_0-1$, we obtain that $||V^{k+1}||^2_{2,n}\leq (1-\frac{1-c}{q(k)})||V^k||^2_{2,n}+4\frac{\sqrt{1-c}}{\sqrt{q(k)}}\lambda_n ||V^k||_{2,n}\leq (1-\frac{1-c}{q(k)})||V^k||^2_{2,n}+\frac{4\sqrt{1-c}}{\eta q(k)}||V^k||^2_{2,n}=\frac{1-c-\psi}{q(k)}||V^k||^2_{2,n},$
where $\psi=\frac{4\sqrt{1-c}}{\eta}$ can be an arbitrarily small constant when $\eta$ is large enough.

Similar to the above inequality, recall the definition of $N(m), R(m) \textrm{ and } \widetilde{N}(m)$. By the argument in Lemma \ref{Lemma:FR}, when $n$ is large enough, $\frac{1}{1-c}||V^m||_{2,n}^2\geq\sum_{k\in \widetilde{N}(m)}(\gamma^{k-1})^2\geq \sum_{k\in \widetilde{N}(m)}\frac{1-c}{q(k-1)}||V^{k-1}||_{2,n}^2-2\sum_{k\in \widetilde{N}(m)} \frac{\sqrt{1-c}}{\sqrt{q(k-1)}}\lambda_n ||V^{k-1}||_{2,n}\geq \sum_{k\in \widetilde{N}(m)}\frac{1-c-\psi}{q(k-1)}||V^{k-1}||_{2,n}^2$.

Thus, following the proof of Lemma \ref{Lemma:FR}, we can treat $1-c-\psi$ as the constant $1-c$ in Lemma \ref{Lemma:FR}, and we obtain:
\begin{equation}\label{eq:Vm dec}
||V^m||_{2,n}^2\precsim ||V^0||_{2,n}^2 (\frac{s}{m+s})^{\zeta^*(c)-\delta-\psi},
\end{equation}
for some small $\delta>0$, and for all $m< \widetilde{m}$. Define $\delta'=\delta+\psi$.

On the other hand, $||V^m||_{2,n}^2\geq (\eta \sqrt{m+s}\lambda_n)^2$ for all $m<\widetilde{m}$. Therefore, combining with (\ref{eq:Vm dec}), we get:
$${\frac{s\log(p)}{n}}\precsim ||V^0||^2_{2,n}(\frac{s}{\widetilde{m}-1+s})^{\zeta^*(c)-\delta'+1},$$
or equivalently,
$\widetilde{m}\precsim s(\frac{s\log(p)}{n ||V^0||^2_{2,n}})^{-\frac{1}{1+\zeta^*(c)-\delta'}}$.

By assumption, $\log(M_0/s)+(\xi+\frac{1}{1+\zeta^*(c)})\log(\frac{s\log(p)}{n ||V^0||_{2,n}^2})>0$ for some $\xi>0$. Thus, asymptotically, $\widetilde{m}=M_0\wedge (m^*+1)<M_0$, i.e., $m^*+1<M_0$.

Therefore, for $\delta'$ small enough, $m^*+1<M_0$. Thus, $m^*\precsim s(\frac{s\log(p)}{n})^{-\frac{1}{1+\zeta^*(c)-\delta'}}$, and $||V^{m^*+1}||_{2,n}^2\leq \eta\sqrt{m^*+1+s}\lambda_n\precsim ||V^0||_{2,n}^{\frac{2}{1+\zeta^*(c)-\delta'}}\left(\frac{s\log(p)}{n}\right)^{\frac{\zeta^*(c)-\delta'}{1+\zeta^*(c)-\delta'}}$, for any small $\delta'>0$.

\end{proof}

\begin{proof}[Proof of Theorem \ref{main}]

At the $(m^*_1+1)^{th}$ step, we have:
$$||U^{m^*_1+1}||^2_{2,n}> (1-c_u\log(p)/n)||U^{m^*_1}||^2_{2,n}.$$
It follows that $(\gamma^{m_1^*})^2<c_u\log(p)/n ||U^{m^*_1}||^2_{2,n}$, while $(\gamma^{m})^2\geq c_u\log(p)/n ||U^{m}||^2_{2,n}$ for all $m<m^*_1$.

Consider the $m^*$ defined in Lemma \ref{Lemma:Bounds} as a reference point.

(a) Suppose $m^*_1< m^*$: By the proof of Lemma \ref{Lemma:Bounds}, $||V^m||^2$ is decreasing when $m\leq m_1^*+1$.

By Lemma \ref{lemma: Zm}, $||U^{m_1^*}||^2_{2,n}\leq \sigma_n^2 +2\frac{\sqrt{m_1^*+s}}{\sqrt{1-c}}\lambda_n ||V^{m_1^*}||_{2,n}$.

It follows that
\begin{equation}\label{eq:ineq-main1}
(\gamma^{m_1^*})^2< c_u\log(p)/n ||U^{m_1^*}||^2_{2,n}<c_u\lambda_n+2c_u\log(p)/n \frac{\sqrt{m_1^*+s}}{\sqrt{1-c}}\lambda_n ||V^{m_1^*}||_{2,n}.
\end{equation}
Now we would like to form a lower bound for $(\gamma^{m_1^*})^2$.

$(\gamma^{m^*_1})^2 = ||U^{m^*_1}||^2_{2,n}-||U^{m^*_1+1}||^2_{2,n} = ||V^{m^*_1}||_{2,n}^2-||V^{m^*_1+1}||_{2,n}^2 - 2\gamma^{m^*_1}<X_{j^{m^*_1}},\epsilon>_n\geq ||V^{m^*_1}||_{2,n}^2-||V^{m^*_1+1}||_{2,n}^2-2\lambda_n|\gamma^{m^*_1}|.$ By inequality (\ref{eq:Vm}), $||V^{m^*_1}||_{2,n}^2-||V^{m^*_1+1}||_{2,n}^2\geq \frac{1-c}{q(m)}||V^{m^*_1}||_{2,n}^2-4\frac{\sqrt{1-c}}{\sqrt{q(m)}}\lambda_n ||V^{m^*_1}||_{2,n}$.

So, $(\gamma^{m^*_1})^2\geq \frac{1-c}{q(m)}||V^{m^*_1}||_{2,n}^2 - 2\lambda_n |\gamma^{m^*_1}| - 4\frac{\sqrt{1-c}}{\sqrt{q(m)}}\lambda_n ||V^{m^*_1}||_{2,n}.$

Consequently, \begin{equation}\label{eq:ineq-main2}(\gamma^{m^*_1})^2 \geq \frac{\sqrt{1-c}}{\sqrt{q(m^*_1)}}||V^{m^*_1}||_{2,n}-4\lambda_n\end{equation}

Plugging inequality (\ref{eq:ineq-main2}) in inequality (\ref{eq:ineq-main1}), it is easy to see that $||V^{m^*_1}||^2_{2,n}\leq K(m^*_1+s)\lambda_n^2\precsim (m^*\wedge s)\lambda_n^2$ for some $K>0$.

By Lemma \ref{Lemma:Bounds}, $(m^*\wedge s)\lambda_n^2\precsim_p ||V^0||_{2,n}^\frac{1}{1+\zeta^*(c)-\delta}\left(\frac{s\log(p)}{n}\right)^{\frac{\zeta^*(c)-\delta}{1+\zeta^*(c)-\delta}}$.

(b) Suppose $m^*_1\geq m^*$: it follows that $(\gamma^m_{j^m})^2\geq c_u\log(p)/n ||U^m||^2_{2,n}$ for all $m<m^*_1$. Since $||U^m||^2_{2,n}$ is a decreasing sequence, for $\delta$ small enough, there exists some $m_2$ such that $||U^m_2||^2_{2,n}> (1-\delta)\sigma_n^2$ for any $m\leq m_2$, and $||U^{m_2+1}||^2_{2,n}\leq (1-\delta)\sigma_n^2$.

For $\delta$ small enough and $m\leq m_2\wedge m^*_1$,  $||V^{m+1}||^2_{2,n}- ||V^m||^2_{2,n} =-(\gamma^m)^2-2\gamma^m<X_{j^m},\epsilon>_n\leq -(\gamma^m)^2 + 2\lambda_n |\gamma^m|$. Since $c_u>4$, so for $\delta$ small enough, $|\gamma^m|^2 \geq c_u\log(p)/n||U^m||^2_{2,n}\geq c_u (1-\delta)\lambda_n^2 \geq 4\lambda_n^2$, so $-(\gamma^m)^2 + 2\lambda_n |\gamma^m|<0$.

Case (b.1): Suppose $m^*_1\leq m_2$:\\ Then, $||V^{m_1^*}||_{2,n}^2\leq ||V^{m^*}||_{2,n}^2\precsim_p ||V^0||_{2,n}^\frac{1}{1+\zeta^*(c)-\delta}\left(\frac{s\log(p)}{n}\right)^{\frac{\zeta^*(c)-\delta}{1+\zeta^*(c)-\delta}}$.

Case (b.2): Suppose $m^*_1 > m_2$: 
We show that this leads to a contradiction.

First of all, we claim that $m_2\geq m^*+1$. We prove this contradiction:

We know that $||U^m||^2_{2,n}=\sigma_n^2+||V^m||^2_{2,n}+2<V^m,\epsilon>_n$. Since $||U^{m_2+1}||^2_{2,n}\leq (1-\delta)\sigma_n^2$, $2<V^{m_2},\epsilon>_n\leq ||V^{m_2}||_{2,n}^2+2<V^{m_2},\epsilon>_n\leq -\delta\sigma_n^2$. So $|2<V^{m_2},\epsilon>_n|\geq \delta\sigma_n^2$.

Suppose $m_2\leq m^*$, it follows that $m_2\leq m^*<M_0$.
 Since we know that $||V^{m_2}||^2_{2,n}$ is decreasing for all $m\leq m^*$, we have $||V^{m_2}||_{2,n}^2\geq ||V^{m^*}||_{2,n}^2\geq \eta(m^*+s)\lambda_n^2$.

Equivalently, $||V^{m^*}||_{2,n}\geq \eta\sqrt{m^*+s}\lambda_n$.

Therefore, $||V^{m_2}||_{2,n}^2+2<V^{m_2},\epsilon>_n\geq ||V^{m_2}||_{2,n}(||V^{m_2}||_{2,n}-2\frac{\sqrt{m_2+s}}{\sqrt{1-c}}\lambda_n)>0$, which is a contradiction to $||V^{m_2}||_{2,n}^2+2<V^{m_2},\epsilon>_n\leq -\delta\sigma_n^2$.

So it must hold that $m_2\geq m^*+1$. Therefore, $||V^{m}||^2_{2,n}\leq ||V^{m^*+1}||^2_{2,n}\leq c_u(m^*+s+1)\lambda_n^2$ for any $m^*+1\leq m\leq m_2$.

We also know that by assumption, $(\gamma^{m})^2\geq c_u(1-\delta)\lambda_n^2$, for any $m\leq m_2<m_1^*$.

Since $||V^m||^2_{2,n}-||V^{m+1}||^2_{2,n}=(\gamma^m)^2-2\gamma^m<X_{j^m},\epsilon>_n\geq c_{u_1}\lambda_n^2>0$, for some constant $c_{u_1}>0$ if $(1-\delta)c_u>2$, it follows that $||V^m||^2_{2,n}\geq ||V^{m+1}||^2_{2,n}$ for $m=m^*,m^*+1,....,m_2-1$. Consequently, $||V^{m_2}||^2_{2,n}\leq ||V^{m^*}||_{2,n}^2$.



By assumption, at the $(m_2+1)^{th}$ step, we know that $||U^{m_2+1}||^2_{2,n}\leq (1-\delta)\sigma_n^2$. It follows that:
\begin{equation}\label{eq:contradict}
|<V^{m_2},\epsilon>_n|\geq \delta'\sigma_n^2,
\end{equation}

for some positive constant $\delta'>0$.

However, $||V^{m_2}||^2_{2,n}\leq ||V^{m^*}||_{2,n}^2$, so $|<V^{m_2},\epsilon>_n|\leq ||V^{m_2}||_{2,n}\sigma_n\leq ||V^{m^*}||_{2,n}\sigma_n\rightarrow 0$, which contradicts (\ref{eq:contradict}).

By collecting all the results in (a), (b).1, (b).2, our conclusion holds.




\end{proof}

\subsection{Proofs for oBA}
\begin{proof}[Proof of Lemma \ref{lemma:lowerboundoga}]

By the sparse eigenvalue condition: $$||V^m_o||_{2,n}^2\geq  (1-c)||\beta_{S^m}||_{2}^2\geq \frac{1-c}{C}||X\beta_{S^m}||_{2,n}^2.$$

Similarly, for $U^m_o$, $||U^m_o||_{2,n}^2=||M_{P^m_o}\varepsilon+M_{P^m_o}(X\beta)||_{2,n}^2=||M_{P^m}\varepsilon||_{2,n}^2+||V^m_o||_{2,n}^2+2<\varepsilon,V^m_o>_n\geq \frac{n-m}{n}\sigma^2+||V^m_o||^{2}_{2,n}-2\lambda_n \sqrt{m+s}||V^m_o||_{2,n}$.
\end{proof}

\begin{proof}[Proof of Lemma \ref{OGA_roc}]

At step $Ks$, if $T^{Ks}\supset T$, then $||V^m_o||_{2,n}^2=||M_{P^m_o}V^m_o||_{2,n}^2=0$. The estimated predictor satisfies: $||x\beta-x\beta^m||_{2,n}\leq 2(1+\eta)\sigma\sqrt{ \frac{s\log(p)}{n}}$ with probability going to 1, where $\eta>0$ is a constant.

Let $A_1=\{T^{Ks}\supset T\}$.

Consider the event $A_1^c$. Then there exists a $j\in T$ which is never picked up in the process at $k=0,1,\ldots,Ks$.

At every step we pick a $j$ to maximize $|<X_j,U^m_o>_{2,n}|=|<X_j,V^m_o>_{2,n}+<X_j,\varepsilon>_2|$.

Let $W^m=X\beta_{S^m}$ and $\widetilde{W}^m=X\alpha^m_{S^m_c}$. Then $||V^m_o||_{2,n}^2\geq (1-c)(||\beta_{S^m}||_2^2+||\alpha_{S^m_c}||^2_2)\geq (1-c)\sum_{j\in S^m}|\beta_{S^m}|^2$.

Also $V^m_o=M_{P^m_o}X\beta_{S^m}$, so $<V^m_o,X\beta_{S^m}>_{2,n}=||M_{P^m_o}X\beta_{S^m}||_{2,n}^2=||X\beta_{S^m}-X_{S^m_c}\zeta||_{2,n}^2\geq (1-c)||\beta_{S^m}||_{2,n}^2$, where $X_{S^m_c}\zeta=P_{X_{S^m_c}} (X\beta_{S^m})$.

Thus, it is easy to see that $<V^m_o,X\beta_{S^m}>_{2,n}=\sum_{j\in S^m}\beta_j <V^m_o,X_j>_{2,n} \geq (1-c)\sum_{j \in S^m}\beta_j^2$.

Thus, there exists some $j^*$ such that $|<V^m_o,X_{j^*}>|\geq (1-c)|\beta_j|\geq cJ$.




We know that the optimal $j_m$ must satisfy: $|<U^m_o,X_{j_m}>_{2,n}|\geq |<V^m_o,X_{j^*}>_{2,n}-<\varepsilon,X_{j^*}>_{2,n}|\geq (1-c)J-\lambda_n$. Thus, $|<V^m_o,X_{j_m}>_{2,n}|> (1-c)J-2\lambda_n$.

Hence, $||V^{m+1}_o||_{2,n}^2=||V^m_o-\gamma_{j_m}X_{j_m}||_{2,n}^2=||V^m_o||_{2,n}^2-2\gamma_{j_m}<U^m_o,X_{j_m}>+2\gamma_j<\epsilon,X_{j_m}>+\gamma_{j_m}^2$

$\leq ||V^m_o||_{2,n}^2-\gamma_{j_m}^2+2\lambda_n|\gamma_{j_m}|\leq ||V^m_o||_{2,n}^2-((1-c)J-\lambda_n)^2+2((1-c)J-\lambda_n)\lambda_n\leq ||V^m_o||_{2,n}^2-((1-c)J)^2+4(1-c)J\lambda_n.$ Consequently, $||V^m_o||_{2,n}^2\leq ||V^m_0||_{2,n}^2-K((1-c)J)^2s+4K(1-c)Js\lambda_n$. Since $||V^0_o||_{2,n}^2\leq CJ'^2s$, so let $K>\frac{(1-c)^2J^2}{CJ'^2}$ and assuming $\lambda_n\rightarrow 0$ would lead to
a $||V^m_o||_{2,n}^2<0$ asymptotically. That said, our assumption that \textquotedblleft there exists a $j$ which is never picked up in the process at $k=0,1,2,\ldots,KS$\textquotedblright\ is incorrect with probability going to 1.

Thus, at time $Ks$, $A_1$ must happen with probability going to 1. And therefore, we know that $||U^m_o||_{2,n}^2=||M_{P^{Ks}_o}\varepsilon||_{2,n}^2\leq \hat\sigma^2= \sigma^2+O_p(\frac{1}{\sqrt{n}})$. By definition of $m^*$, $m^*\leq Ks$. Therefore, $||V^m_o||_{2,n}^2\leq ||U^m_o||_{2,n}^2-\sigma^2+m^*\lambda_n=O_p(s\lambda_n^2)$.

\end{proof}

\subsection{Proofs for post-BA}
\begin{proof}[proof of Lemma \ref{lemma:post-boosting}]

It is sufficient to show that $T^0\subset T^{m^*}-T^0$ for $m^*=Ks$ with $K$ being large enough. If there exists a $j\in T^0$ which is never revisited at steps $1,2,...,m^*$, then in each step, we can choose the variable $j$:
By assumption A.2, the optimal step size $\gamma_j^m:=<U^m,X_j>_n=<\epsilon,X_j>_n+<V^m,X_j>_n$ must satisfy: $|\gamma_j^m|\geq_p \sqrt{1-c}|\beta_j|-\lambda_n> \sqrt{1-c}J(1-o(1))$. Hence, each step $||U^m||_{2,n}^2$ must decrease for at least ${(1-c)}J^2(1-\delta)^2$ for any $\delta>0$ and $n$ large enough. However, $||U^0||^2_{2,n}\leq (1+c)s(J')^2$, which implies that $m^*=Ks\leq \frac{(1+c)s(J')^2}{{(1-c)}J^2(1-\delta)^2}$, i.e, $K\leq \frac{(1+c)(J')^2}{{(1-c)}J^2(1-\delta)^2}$.

So for any $K>\frac{(1+c)(J')^2}{(1-c)J^2}$, as $n\rightarrow\infty$ and for $\delta$ small enough, $m^*Ks> \frac{(1+c)s(J')^2}{{(1-c)}J^2(1-\delta)^2}$, which leads to a contradiction.

Thus, for any $K>\frac{(1+c)(J')^2}{(1-c)J^2}$, all variables in $T^0$ must be revisited at steps $1,2,...,m^*$ with probability going to 1. The rest of the results simply follow $T^0\subset T^{m^*}-T^0$.
\end{proof}

\vskip 0.2in
\bibliography{Literatur_NR}

\end{document}

%% file: Simulation_results_edit.tex
\begin{table}[ht]
\tiny
\centering
\caption{Simulation results: sparse, iid design (Boosting)}
\vspace{0.2cm}
\begin{tabular}{rrrrrrrrrrrrrrrrr}
  \hline 
  n & p & BA-oracle & BA-Ks & BA-our & p-BA-oracle & p-BA-Ks & p-BA-our & oBA-oracle & oBA-Ks & oBA-our & \\ 
  \hline
  100 & 100 & 0.44 & 0.69 & 0.66 & 0.12 & 0.58 & 0.43 & 0.12 & 0.82 & 0.54  \\ 
    100 & 200 & 0.48 & 0.85 & 1.28 & 0.14 & 0.77 & 1.65 & 0.12 & 1.00 & 0.60   \\ 
    200 & 100 & 0.15 & 0.29 & 0.26 & 0.05 & 0.25 & 0.21 & 0.05 & 0.34 & 0.20 \\ 
    200 & 200 & 0.20 & 0.41 & 0.35 & 0.06 & 0.31 & 0.21 & 0.06 & 0.41 & 0.24  \\ 
    400 & 100 & 0.07 & 0.13 & 0.10 & 0.03 & 0.11 & 0.08 & 0.03 & 0.14 & 0.09  \\ 
    400 & 200 & 0.09 & 0.19 & 0.16 & 0.03 & 0.16 & 0.12 & 0.02 & 0.21 & 0.14  \\ 
   \hline
\end{tabular}
\end{table}

\begin{table}[ht]
\tiny
\centering
\caption{Simulation results: sparse, iid design (Lasso)}
\vspace{0.2cm}
\begin{tabular}{rrrrrrrrrrrrrrrrr}
  \hline 
  n & p &  LASSO  & p LASSO  & Lasso-CV & p-Lasso-CV \\ 
  \hline
  100 & 100  & 0.88 & 0.70 & 0.54 & 0.94 \\ 
    100 & 200 &  1.02 & 1.30 & 0.72 & 1.02 \\ 
    200 & 100 &  0.29 & 0.28 & 0.22 & 0.43 \\ 
    200 & 200 &  0.37 & 0.39 & 0.30 & 0.44 \\ 
    400 & 100 &   0.13 & 0.11 & 0.09 & 0.18 \\ 
    400 & 200 &  0.16 & 0.20 & 0.14 & 0.27 \\ 
   \hline
\end{tabular}
\end{table}

\begin{table}[ht]
\tiny
\centering
\caption{Simulation results: sparse, correlated design (Boosting)}
\vspace{0.2cm}
\begin{tabular}{rrrrrrrrrrrrrrrrrr}
\hline
   n & p & BA-oracle & BA-Ks & BA-our & p-BA-oracle & p-BA-Ks & p-BA-our & oBA-oracle & oBA-Ks & oBA-our  \\ 
  \hline
  100 & 100 & 1.40 & 1.70 & 1.90 & 0.55 & 1.02 & 1.31 & 0.44 & 0.96 & 1.36  \\ 
    100 & 200 & 3.02 & 2.80 & 2.85 & 1.65 & 2.29 & 2.48 & 1.25 & 1.44 & 1.96  \\ 
    200 & 100 & 0.41 & 0.48 & 0.54 & 0.07 & 0.12 & 0.16 & 0.07 & 0.35 & 0.24  \\ 
    200 & 200 & 0.53 & 0.60 & 0.63 & 0.06 & 0.15 & 0.19 & 0.06 & 0.42 & 0.25  \\ 
    400 & 100 & 0.16 & 0.28 & 0.19 & 0.02 & 0.04 & 0.08 & 0.02 & 0.14 & 0.09  \\ 
    400 & 200 & 0.17 & 0.23 & 0.21 & 0.03 & 0.04 & 0.10 & 0.03 & 0.17 & 0.10 \\ 
   \hline
\end{tabular}
\end{table}

\begin{table}[ht]
\tiny
\centering
\caption{Simulation results: sparse, correlated design (Lasso)}
\vspace{0.2cm}
\begin{tabular}{rrrrrrrrrrrrrrrrrr}
\hline
   n & p & LASSO & p-LASSO  & Lasso-CV & p-Lasso-CV \\ 
  \hline
  100 & 100 &    2.63 & 1.35 & 0.97 & 1.37 \\ 
    100 & 200 &  2.96 & 2.04 & 1.63 & 2.38 \\ 
    200 & 100 & 1.10 & 0.23 & 0.33 & 0.57 \\ 
    200 & 200 & 1.64 & 0.38 & 0.49 & 0.87 \\ 
    400 & 100 & 0.38 & 0.10 & 0.13 & 0.23 \\ 
    400 & 200 &  0.36 & 0.15 & 0.16 & 0.31 \\ 
   \hline
\end{tabular}
\end{table}

\begin{table}[ht]
\tiny
\centering
\caption{Simulation results: polynomial, iid design (Boosting)}
\vspace{0.2cm}
\begin{tabular}{rrrrrrrrrrrrrrrrrr}
  \hline
  n & p & BA-oracle & BA-Ks & BA-our & p-BA-oracle & p-BA-Ks & p-BA-our & oBA-oracle & oBA-Ks & oBA-our  \\ 
  \hline
  100 & 100 & 0.37 & 0.81 & 0.58 & 0.36 & 1.09 & 0.64 & 0.37 & 1.23 & 0.73  \\ 
    100 & 200 & 0.44 & 1.04 & 1.38 & 0.43 & 1.32 & 1.85 & 0.45 & 1.39 & 0.74  \\ 
    200 & 100 & 0.26 & 0.40 & 0.34 & 0.26 & 0.47 & 0.37 & 0.26 & 0.49 & 0.39  \\ 
    200 & 200 & 0.27 & 0.54 & 0.39 & 0.27 & 0.63 & 0.42 & 0.28 & 0.61 & 0.44  \\ 
    400 & 100 & 0.17 & 0.19 & 0.19 & 0.17 & 0.21 & 0.20 & 0.18 & 0.22 & 0.20  \\ 
    400 & 200 & 0.18 & 0.30 & 0.26 & 0.18 & 0.33 & 0.28 & 0.17 & 0.34 & 0.28 \\ 
   \hline
\end{tabular}
\end{table}

\begin{table}[ht]
\tiny
\centering
\caption{Simulation results: polynomial, iid design (Lasso)}
\vspace{0.2cm}
\begin{tabular}{rrrrrrrrrrrrrrrrrr}
  \hline
  n & p & LASSO & p-LASSO  & Lasso-CV & p-Lasso-CV \\ 
  \hline
  100 & 100 &  0.43 & 0.83 & 0.45 & 0.84 \\ 
    100 & 200 &   0.50 & 1.06 & 0.54 & 0.76 \\ 
    200 & 100 &   0.30 & 0.34 & 0.26 & 0.47 \\ 
    200 & 200 & 0.34 & 0.52 & 0.33 & 0.52 \\ 
    400 & 100 &  0.19 & 0.19 & 0.15 & 0.24 \\ 
    400 & 200 &  0.21 & 0.31 & 0.20 & 0.38 \\ 
   \hline
\end{tabular}
\end{table}

\begin{table}[ht]
\tiny
\centering
\caption{Simulation results: polynomial, correlated design (Boosting)}
\vspace{0.2cm}
\begin{tabular}{rrrrrrrrrrrrrrrrr}
  \hline
  n & p & BA-oracle & BA-Ks & BA-our & p-BA-oracle & p-BA-Ks & p-BA-our & oBA-oracle & oBA-Ks & oBA-our  \\ 
  \hline
  100 & 100 & 0.23 & 0.68 & 0.46 & 0.22 & 0.91 & 0.49 & 0.22 & 1.22 & 0.51  \\ 
    100 & 200 & 0.26 & 0.87 & 1.02 & 0.24 & 1.10 & 1.42 & 0.24 & 1.46 & 0.66  \\ 
    200 & 100 & 0.19 & 0.37 & 0.28 & 0.15 & 0.44 & 0.26 & 0.14 & 0.49 & 0.24  \\ 
    200 & 200 & 0.22 & 0.49 & 0.35 & 0.20 & 0.56 & 0.34 & 0.20 & 0.61 & 0.34  \\ 
    400 & 100 & 0.13 & 0.20 & 0.17 & 0.10 & 0.20 & 0.15 & 0.09 & 0.20 & 0.15 \\ 
    400 & 200 & 0.14 & 0.23 & 0.20 & 0.11 & 0.24 & 0.17 & 0.11 & 0.28 & 0.17  \\ 
   \hline
\end{tabular}
\end{table}

\begin{table}[ht]
\tiny
\centering
\caption{Simulation results: polynomial, correlated design (Lasso)}
\vspace{0.2cm}
\begin{tabular}{rrrrrrrrrrrrrrrrr}
  \hline
  n & p   & LASSO & p-LASSO &  Lasso-CV & p-Lasso-CV \\ 
  \hline
  100 & 100 &    0.33 & 0.53 & 0.33 & 0.55 \\ 
    100 & 200 &  0.34 & 0.93 & 0.36 & 0.55 \\ 
    200 & 100 &  0.27 & 0.31 & 0.23 & 0.41 \\ 
    200 & 200 &  0.28 & 0.47 & 0.29 & 0.46 \\ 
    400 & 100 &   0.17 & 0.18 & 0.14 & 0.24 \\ 
    400 & 200 &   0.16 & 0.24 & 0.15 & 0.29 \\ 
   \hline
\end{tabular}
\end{table}